\documentclass[11pt]{article}

\usepackage[utf8]{inputenc}
\usepackage[margin=1in]{geometry}
\usepackage[numbers]{natbib}
\usepackage{floatpag}
\usepackage{float}
\usepackage{wrapfig}
\newfloat{algorithm}{t}{lop}

\usepackage{microtype}
\usepackage{graphicx}
\usepackage{booktabs} 

\usepackage{hyperref}

\usepackage{amsmath}
\usepackage{amssymb}
\usepackage{mathtools}
\usepackage{amsthm}
\usepackage{xcolor}
\usepackage{algorithm}
\usepackage{algorithmic}

\usepackage[capitalize,noabbrev]{cleveref}

\usepackage{hyperref}

\usepackage{url}            
\usepackage{amsfonts}       
\usepackage{nicefrac}       

\usepackage{xcolor}         
\usepackage{bm}

\usepackage{dsfont}
\usepackage{float}
\usepackage{subcaption}
\usepackage{times}
\usepackage{tikz}
\usepackage{multirow}
\usepackage{tabularx}
\usepackage{wrapfig}

\usepackage{colour-blind}

\theoremstyle{plain}
\newtheorem{theorem}{Theorem}[section]

\newtheorem{lemma}[theorem]{Lemma}
\newtheorem{corollary}[theorem]{Corollary}
\theoremstyle{definition}
\newtheorem{definition}[theorem]{Definition}
\newtheorem{assumption}{Assumption}

\theoremstyle{remark}
\newtheorem{remark}[theorem]{Remark}

\newtheorem*{theorem*}{Theorem}
\newtheorem*{lemma*}{Lemma}
\newtheorem*{corollary*}{Corollary}

\newcommand{\calA}{\ensuremath{\mathcal{A}}}
\newcommand{\calB}{\ensuremath{\mathcal{B}}}
\newcommand{\calC}{\ensuremath{\mathcal{C}}}
\newcommand{\calD}{\ensuremath{\mathcal{D}}}

\newcommand{\calP}{\ensuremath{\mathcal{P}}}

\newcommand{\calR}{\ensuremath{\mathcal{R}}}
\newcommand{\calS}{\ensuremath{\mathcal{S}}}

\newcommand{\calU}{\ensuremath{\mathcal{U}}}

\newcommand{\boldw}{\ensuremath{\mathbf{w}}}
\newcommand{\boldv}{\ensuremath{\mathbf{v}}}
\newcommand{\boldr}{\ensuremath{\mathbf{r}}}

\newcommand{\bA}{\ensuremath{\mathbf{A}}}
\newcommand{\bC}{\ensuremath{\mathbf{C}}}
\newcommand{\bY}{\ensuremath{\mathbf{Y}}}
\newcommand{\bSig}{\ensuremath{\mathbf{\Sigma}}}

\newcommand{\avgw}{\ensuremath{\overline{\boldw}}}
\newcommand{\avgv}{\ensuremath{\overline{\boldv}}}
\newcommand{\avgz}{\ensuremath{\overline{\mathbf{z}}}}

\newcommand{\argmin}{\ensuremath{\text{argmin}}}

\newcommand{\Expec}[2]{\ensuremath{\mathbb{E}_{#1}\left[ #2 \right]}}

\newcommand{\tr}{\ensuremath{\mathbf{Tr}}}

\newcommand{\feast}{\textsc{F3ast} }
\newcommand{\feastcomma}{\textsc{F3ast,} }

\newcommand{\vargrad}{\textsc{F3ast} }

\newcommand{\fedavg}{\textsc{FedAvg} }

\newcommand{\poc}{\textsc{PoC} }

\newcommand{\fedopt}{\textsc{FedOpt}\,\,}
\newcommand{\fedadam}{\textsc{FedAdam}\,\,}

\renewcommand{\algorithmicrequire}{\textbf{Input:}}

\renewcommand{\algorithmicensure}{\textbf{Output:}}

\makeatletter 
\def\clineThicknessColor#1#2#3{\@ClineThicknessColor#1\@nil{#2}{#3}}
\def\@ClineThicknessColor#1-#2\@nil#3#4{%
    \omit
    \@multicnt#1%
    \advance\@multispan\m@ne
    \ifnum\@multicnt=\@ne\@firstofone{&\omit}\fi
    \@multicnt#2%
    \advance\@multicnt-#1%
    \advance\@multispan\@ne
    \color{#4}
    \leaders\hrule\@height#3\hfill
    \leaders\hrule\@height#3\hfill
    \cr}
\makeatother 
\hypersetup{
    colorlinks=true,
    linkcolor=blue,
    urlcolor=blue,
    citecolor=blue
}

\newcommand{\narxiv}[1]{}

\renewenvironment{abstract}
 {\small
  \begin{center}
  \bfseries \abstractname\vspace{-.5em}\vspace{0pt}
  \end{center}
  \list{}{
    \setlength{\leftmargin}{6mm} 
    \setlength{\rightmargin}{\leftmargin}
  }
  \item\relax}
 {\endlist}

\begin{document}

\title{Federated Learning Under Intermittent Client Availability and Time-Varying Communication Constraints}
\author{M\'onica Ribero\thanks{Dept. of Electrical and Computer Engineering, The University of Texas at Austin.  } \and Haris Vikalo \footnotemark[1] \and Gustavo De Veciana \footnotemark[1]}
\maketitle

\begin{abstract}
Federated learning systems facilitate training of global models in settings where potentially heterogeneous data is distributed across a large number of clients. Such systems operate in settings with intermittent client availability and/or time-varying communication constraints. As a result, the global models trained by federated learning systems may be biased towards clients with higher availability. We propose \feast \!\!, an unbiased algorithm that dynamically learns an availability-dependent client selection strategy which asymptotically minimizes the impact of client-sampling variance on the global model convergence, enhancing performance of federated learning. The proposed algorithm is tested in a variety of settings for intermittently available clients under communication constraints, and its efficacy demonstrated on synthetic data and realistically federated benchmarking experiments using CIFAR100 and Shakespeare datasets. We show up to 186\% and 8\% accuracy improvements over \fedavg,  and 8\%  and 7\% over \fedadam on CIFAR100 and Shakespeare, respectively. \end{abstract}

\section{Introduction}
\label{sec:intro}

Federated learning (FL) has emerged as an attractive framework for training machine learning models in settings where the data is distributed among remote clients and must remain local due to privacy concerns and/or resource constraints. In the original Federated Averaging algorithm (\textsc{FedAvg}) \cite{mcmahan17a}, as well as more recent approaches including \textsc{Scaffold} \cite{karimireddy2020scaffold}, Federated Adaptive Optimization \cite{reddi2020adaptive}, \textsc{FedDyn} \cite{acar2021federated} and FedProx \cite{li2018federated}, a server selects a random subset of clients and which will participate in updating a global model by training on local data. The server aggregates the clients' updates to produce a new global model, broadcasts it to the clients, and a new round of training begins; this procedure is repeated until convergence. 

One of the biggest gaps between theory and practice of FL is due to biased sampling of clients caused by their on-and-off availability and communication constraints \cite{bonawitz2019towards, eichner2019semi, WCXJ21}. 
For example, in cross-device settings including mobile device systems \cite{mcmahan17a}, a vast number of client devices \cite{kairouz2019advances} with limited communication and power resources \cite{bonawitz2019towards} \textit{intermittently} connects to a central server to help optimize a global objective. Existing FL algorithms typically ignore intermittency and assume that the participating client devices are always available and thus can be tasked with performing a model update at any time \cite{mcmahan17a, cho2020client, li2019convergence, acar2021federated}. If not addressed by the system design, time-varying communication constraints and intermittent client availability (due to battery and other device-specific limitations) may cause significant degradation of the learned model performance \cite{eichner2019semi, ribero2020communication, cho2020client}. 




To illustrate the potential severity of the problem described above, and preview the contributions of this paper, consider the following simple example. Let $c_1$ and $c_2$ be two clients with distinct data distributions. A server aims to optimize the function $F(\boldw) =p_1F_1(\boldw)+p_2F_2(\boldw)$ over $\mathbb{R}^p$, where $F_1$ and $F_2$ denote the loss functions at clients $c_1$ and $c_2$, respectively, and for simplicity $p_1 = p_2 = 1/2$. We shall consider a model for the clients' intermittent availability characterized by the joint distributions given in \cref{tab:availabilities-1},
where $A_i$ is a binary random variable indicating whether client $c_i$ is available. In this model, clients' availabilities in a given round are independent, with $P(A_1=1)=0.375$ while $P(A_2=1)=0.8$. 
Note the client availabilities are assumed to be independent across rounds. Suppose there is a communication constraint which restricts the server to sampling at most a single client each round. The server must thus choose a possibly client state dependent policy for selecting the clients in each round. Each such policy would achieve certain long-term client participation rates across the rounds, denoted $\boldr = (r_1,r_2)$. For example, under the model in Table~1, the set of achievable long-term participation rates across all possible policies is given by the region $\calR_1$ shown in \cref{fig:constraint-region-1}. Given the communication constraints, it is not possible to achieve the full client participation rates of $\boldr =(1,1)$ because clients are not always available and we can only sample one client in each round. However, $\boldr^{a}=(0.375,0)$ is achievable by using the state-dependent deterministic policy which selects $c_1$ whenever $c_1$ is available and never selects $c_2$. Alternatively, $\boldr^{b}=(0.375,0.5)$ is also achievable by selecting $c_1$ whenever $c_1$ is available, and $c_2$ when only $c_2$ is available. A naive selection policy  that samples from available clients with probability proportional to $p_i=\frac{1}{2}$ in hope of achieving ``ideal" participation rate $(\frac{1}{2},\frac{1}{2})$ \cite{li2019convergence} would actually result in client $c_1$ participating at a rate of  
$$
    r^c_1 = P( A_1 =1, A_2 =0) + \frac{P(A_1 = A_2 =1 )}{2}= 0.225.
$$
Analogously, the long-term participation rate of client $c_2$ under the same naive selection policy is $r^c_2=0.65$. 

As demonstrated in \cref{sec:model} (\cref{thm:importance_convergence}), improper choice of the long-term participation rate $\boldr$ injects bias and variance into the global model. 
Therefore, selecting an ``appropriate" rate is of fundamental importance; yet, as illustrated above, intermittent client availability and communication constraints present several previously overlooked challenges: (i) determining the long-term participation rate $\boldr^* \in \calR$ which is best in terms of its impact on the convergence of federated learning, and 
(ii)  design of a client selection policy that achieves rate $\boldr^*$. 
These are particularly demanding because the (possibly correlated) clients' availability patterns are unknown and, therefore, $\calR$ is unknown. 

\begin{table}
\caption{Client availability model: The availability is independent across time and clients.}
\vspace{-0.05in}
\label{tab:availabilities-1}
\centering
\begin{tabularx}{0.45\textwidth}{|X|X|X|X|}
\hline
          & $A_2=1$ & $A_2=0$ & Marginal \\ \hline
$A_1=1$     & 0.3   & 0.075   &  0.375     \\ \hline
$A_1=0$ & 0.5 & 0.125     & 0.625     \\ \hline
Marginal & 0.8 & 0.2     &           \\ \hline
\end{tabularx}
\end{table}


\begin{figure}
    \centering
    \begin{tikzpicture}
    \draw[->, thick] (0,-0.3) -- (0,2.3);
    \draw[->, thick] (-0.3,0) -- (1.3,0);
    \draw[orange, ultra thick] (0,0) -- (0,1.6)--(0.15,1.6) --(0.75,1.) --(0.75,0) -- cycle node[right] at (0.05,0.8) {$\calR_1$};
    \filldraw[blue] (1,1) circle (1.5pt) node[anchor=west]{\footnotesize Ideal rate};
    \node at (1.6 ,0.65) {\tiny $(\frac{1}{2},\frac{1}{2})$};
    \filldraw[black] (0.75,0) circle (1.5pt) node[above right] at (0.75,0)
    {\footnotesize $\boldr^{a}$};
    \filldraw[black] (0.75,1.) circle (1.5pt) node[above] at (0.75,1.)
    {\footnotesize $\boldr^{b}$};
    \filldraw[black] (0.45,1.3) circle (1.5pt) node[above] at (0.45,1.3)
    {\footnotesize $\boldr^{c}$};
    \draw[gray, ultra thin] (0.8,-0.07) -- (0.8,0.07) node[below] at (0.8,-0.1) {\footnotesize $r_1$};
    \draw[gray, ultra thin] (-0.07,1.6) -- (0.07,1.6) node at (-0.3,1.6) {\footnotesize $r_2$};
    \end{tikzpicture}
    \caption{The region of achievable long-term participation rates under the client availability model in Table~1.}
    \label{fig:constraint-region-1}
\end{figure}
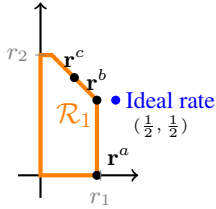

\vspace{-0.05in}
\paragraph{The main contribution of this paper: Learning to sample clients with intermittent client availability.} In this paper we introduce \feast\!\!, a federated learning algorithm that additionally learns and adapts to the \textit{unknown} statistics of client availability and the variability in the FL system capacity (i.e., adapts to time-varying communication constraints). \feast\!\! is shown to be asymptotically optimal (see \cref{thm:sampling_rate_convergence}) 
as its long-term participation rate converges to the value minimizing a bound on the global model variance over the space of achievable rates. Remarkably, \feast accomplishes this \textit{with no prior knowledge of the communication constraints or clients' availability models}. To our knowledge, this is the first work to \textit{formally} address client intermittency and system capacity variability in federated learning, and the first work to propose a method to \textit{learn how to select clients} while pursuing a shared (global) model within the federated learning framework.

\vspace{-0.05in}


\paragraph{Extensive experimentation on realistic tasks and data.} \feast is tested on three benchmark datasets: (i) Synthetic(1,1) \cite{shamir2014communication}, a widely used heterogeneous synthetic dataset for softmax regression \cite{li2018federated}; (ii) a realistically federated version of CIFAR100 \cite{reddi2020adaptive}; and (iii) Shakespeare \cite{mcmahan17a}. We demonstrate that in learning highly non-linear models \feast exhibits more stable convergence and considerably higher accuracy than state-of-art algorithms. Moreover, \feast\!\!'s selection and aggregation method is readily combined with the existing optimization techniques designed to address system's constraints, allowing those methods to take advantage of \feast\!\!'s policies: experiments confirm that incorporating \feast reduces bias of algorithms that do not compensate for client selection uncertainties, and demonstrate much more stable descent trajectories to the optimum even in highly time-varying environments. 
\section{Background and Related Work}
\label{sec:background}

\subsection{Federated Learning}
Given a set $\calU$ with $N$ clients, each having $n_k$ data samples, a federated learning system is concerned with solving
\begin{equation}
    \label{eq:objective_function}
    \min_{\boldw \in \mathbb{R}^{^p}}  \mathbb{E}_{k\sim \calP}[F_k(\boldw)],
\end{equation}
where $ F_k(\boldw) = \mathbb{E}_{\xi \sim \calD_k}[f_k(\boldw; \xi)]$
denotes the loss function of client $k$ and $\calP$ is the distribution over users. 
The generalized \fedavg algorithm, \fedopt \cite{reddi2020adaptive}, interactively learns the global model by \textit{randomly selecting} at time $t$ a subset of clients $S_t$ to locally optimize their objective function\footnote{E.g., the original \fedavg algorithm coordinates $E$ epochs over training data.} starting from the initial model $\avgw^t$, and communicate their updates $\boldv_k^{t+1}$ to the server. Then, the server \textit{aggregates} the received updates $(\boldv_k^{t+1})_{k \in S_t}$ to produce the global update $\Delta^{t+1}$ and generate a new global model $\avgw^{t+1}$. For the remainder of the paper we refer to the process of going from $\avgw^t$ to $\avgw^{t+1}$ as a (single) round of FL. The process is repeated with the aim of finding accurate global model. Heterogeneity due to generally non-i.i.d. and unbalanced data available to different clients emerges as one of the main challenges in federated learning, and thus the choice of a client sampling scheme heavily impacts convergence of the global model. Several authors have addressed this problem, proposing different optimizers \cite{reddi2020adaptive} and client sampling strategies \cite{cho2020client, li2019convergence}, but all showed convergence (or near convergence) only under the strong assumption of being able to sample any client at any time. A technique that deals with heterogeneity by allowing each client to learn a personalized model was proposed in \cite{smith2017federated}.

Several approaches to addressing communication constraints in FL systems have recently been proposed in literature, including strategies aiming to reduce client communication rate via model compression \cite{tang2019doublesqueeze,konevcny2016federated,suresh2017distributed,konevcny2018randomized,alistarh2017qsgd,horvath2019natural}. These are orthogonal to our work since we focus on settings where the clients are intermittently available. 

\paragraph{Client availability.} We assume that, at any time, the set of available clients is random and non-empty, and that the constraint on the number of clients selected to train and provide model updates to the server generally varies over time.
Note that, in such scenarios, applying recent stateful optimization techniques such as \textsc{Scaffold} \cite{karimireddy2020scaffold} and \textsc{FedDyn} \cite{acar2021federated} is challenging due to hardware constraints and high number of participating devices \cite{WCXJ21}.
Prior work has investigated client availability under restrictive conditions such as block-cyclic data characteristics \cite{eichner2019semi, cho2020client}, assumed i.i.d. availability across clients that act as stragglers \cite{li2018federated}, or produced biased models \cite{xia2020multi}. 
However, availability is much more difficult to model in practice. Although certain patterns such as day/night are cyclic, there also exist various other non-cyclic client or cluster-specific patterns that affect some clients more than others, including access to a power source and the available communication bandwidth. 

\subsection{Client sampling and averaging}
Since the number of clients in federated learning systems can be extremely large \cite{hard2018federated, yang2018applied}, only a relatively small subset of them is tasked with training in each round. Data heterogeneity and communication constraints have inspired several strategies for selecting clients from the available pool.
Some of those techniques take into account the proportion of data at each client, which we denote by $p_k$ \cite{li2019convergence}. Alternative strategies apply active learning ideas to client selection and select those that are more promising according to some metric, e.g., choose clients with the largest magnitude of the updates \cite{hsieh2017gaia, chen2018lag, singh2019sparq, ribero2020communication, rizk2020federated} or those with the highest loss \cite{cho2020client}.

Another line of related prior work has been focused on investigating model aggregation strategies. In \cite{ribero2020communication}, the authors assume that stochastic optimization updates approximately follow a stationary stochastic process, and cast the model aggregation as an estimation problem. An alternative aggregation strategy is to form an unweighted average of the updates \cite{cho2020client}; however, this leads to a biased model and large variance. Other approaches trade communication and memory for stable convergence \cite{karimireddy2020scaffold}, or replace missing updates with the previous model \cite{chen2018lag, singh2019sparq}. In the centralized setting, importance sampling has been used to optimally aggregate SGD updates \cite{zhao2015stochastic}, \cite{rizk2020federated}. 

Previous works on client sampling in FL systems do not provide formal convergence guarantees for settings where the clients are intermittently available.
\citet{cho2020client,eichner2019semi} study effects of cyclically alternating client availability and propose a sampling strategy empirically shown to improve over random sampling; however, the resulting models may be biased and their performance under non-cyclic availability patterns is unclear. 

Alternatively, asynchronous methods address client selection under system heterogeneity  with a fixed selection policy determined by clients training speed: updates are incorporated individually as they arrive at the server \cite{CNSR20} or, when there are privacy concerns, they are aggregated in buffers \cite{NMZYRMH21}. 
\section{Methods}
\label{sec:model}

In this section we present and analyze a novel framework for selecting and aggregating intermittently available clients in federated learning systems that operate under time-varying communication constraints. In such settings, the contribution each client makes to the federated averaging process depends on how often the client is selected to provide an update
-- i.e., on the {\em long-term client participation rate}. We start by characterizing the set $\calR$ of achievable long-term client participation rates subject to communication and client availability constraints.
Then, we introduce \feastcomma an algorithm that dynamically learns clients' long-term participation rate and improves the convergence of federated learning by reducing the model bias and minimizing variance introduced by sampling intermittently available clients. The omitted proofs can be found in the appendix. 

\subsection{Preliminaries}


\paragraph{Communication constraints and intermittent client availability.}
Consider a FL system in which a random subset of clients $\bA_t$ is available/responsive at time $t$; here $(\bA_1,\bA_2,\dots) =(\bA_t)_t$ form a discrete-time stochastic process with a finite state space $\calA = 2^\calU$, i.e., the collection of all possible subsets of the set of users $\calU$. Communication constraints restrict the possible subsets of clients that can be chosen to participate during a training round; we let $\bC_t$ denote the (random) collection of the available clients sets that meet communication constraints at time $t$ and denote its state space by $\mathcal{C}$; given $\bA_t = A$, a realization $\bC_t = C$ corresponds to a collection of subsets of $A$, i.e., $C \subseteq 2^{A}$. For convenience, in the remainder of the paper we refer to $\bC_t$ as the system {\em configuration}.

To illustrate the use of the introduced notation, consider the communication-constrained setting where the number of clients allowed to participate in training round $t$ is no more than (possibly random) $K_t$. Given a realization of the set of clients available at time $t$, $\bA_t = A$, and the aforementioned communications constraint $K_t = k \in \mathbb{N}$, the collection of feasible client sample sets $S$ that the server may choose to include in the training round  is
\[
C = \{S \subset A : |S| \leq k \}.
\]
If the communication constraints are not time-varying, i.e., if $K_t=k$ 
almost surely for all $t$, we are back in the traditional \fedavg 
cross-device setting.

\begin{assumption}
\label{assum:constraint_process}
The sequence of random collections of feasible client sampling  sets $(\bC_t )_t $ forms a discrete-time irreducible Markov chain with a finite state space $\calC \subseteq 2^{\calU}$ and a stationary distribution $\pi = (\pi(C) , C\in \calC)$.
\end{assumption}

\cref{assum:constraint_process} significantly relaxes assumptions typically made when analyzing convergence of state-of-the-art FL algorithms, e.g., the much stronger assumptions of all users having unlimited availability \cite{mcmahan17a, rabi2012adaptive, mohri2019agnostic, li2018federated, acar2021federated, karimireddy2020scaffold} or a deterministic block-cyclic availability \cite{eichner2019semi}. \cref{assum:constraint_process} captures various realistic settings including that of home devices available with a given probability - not necessarily uniform across clients - throughout the day. Our experimental results showcase that in several realistic settings which meet \cref{assum:constraint_process}, our method provides significant performance improvements while other techniques fail to adapt to unknown availability models.

\paragraph{Static configuration-dependent client sampling policies.} Communication constraints and intermittent client availability restrict the space of admissible long-term client participation rates. Indeed, it is unrealistic to expect being able to sample an arbitrary collection of $k$ clients at time $t$ with pre-specified probabilities $\calP = (p_i :i=1 \ldots N)$ since some of those $k$ clients may be unavailable, and/or time horizon is not long enough to achieve certain rate (see examples in \cref{sec:intro}). 
To characterize {\em achievable} long-term participation rates we introduce the following concepts.

Recall that for a given configuration of communication constraints and client availabilities there exists an associated collection of feasible client sample sets $C$ that a sampling policy can choose from. We define a {\em static configuration-dependent client sampling policy} as follows. 
\begin{definition}
\label{def:configuration-dependent}
For each $C$, let $f_{C,S}\geq 0$ denote the probability of selecting the subset of clients $S \in C$, where $\sum_{\calS \in C} f_{C,S}=1$. If we denote $\bm{f}_C = (f_{C,S}, S \in C)$, then $\bm{f} := (\bm{f}_C, C \in \calC)$ specifies a static configuration-dependent sampling policy selecting clients over different communication/availability configurations.
\end{definition}

Let ${\cal F}$ denote the set of possible static configuration-dependent client sampling policies. 
Under the above model, the long-term client participation rate can be expressed as
\begin{equation}
\label{eq:client-rate}
\boldr^{\bm{f}} = \sum_{C \in \calC}\pi(C)\sum_{S \in C} f_{_{C,S}}\mathds{1}_{S},
\end{equation}
where $\mathds{1}_{S}$ is an $N$-dimensional binary indicator vector whose $i^{th}$ entry is $1$ if the $i^{th}$ client is in $S$, and is $0$ otherwise.
One can interpret the $i^{th}$ component of vector $\boldr^{\bm{f}}$ as the fraction of time the $i^{th}$ client is selected by the server.


Finally, we define the long-term client participation rate region as the set of all possible long-term participation rate vectors $\boldr^{\bm{f}}$, i.e., $\calR := \{\boldr^{\bm{f}} | \bm{f} \in {\cal F} \}$.

\begin{lemma}
\label{lemma:constraint_region}
The long-term client participation rate region $\calR = \{\boldr^{\bm{f}} | \bm{f} \in {\cal F} \}$ is a subset of the simplex in the $N$-dimensional Euclidean space, and a closed convex set. 
\end{lemma}

\begin{proof}
The lemma follows from the fact that $\calR$ is a linear image of all possible $\bm{f}$, a closed bounded convex set. 
\end{proof}

\subsection{\feast\!: Minimizing the sampling variance}
\label{sec:feast}
Here we formally introduce an algorithm that learns a client selection policy which ensures that the resulting long-term client participation rate converges to a value minimizing 
\begin{equation}
    H(\boldr):= \begin{cases} 
    \vspace{-0.05in}
            \sum_{k=1}^N\frac{p_k}{r_k} & \text{client availability is }\\
            & \text{positively correlated,} \\
            \sum_{k=1}^N\frac{p^2_k}{r_k} & \text{ otherwise}. \\
                \end{cases}
\end{equation}

It is readily shown that $H(\boldr)$ bounds the variance induced in the global model by the selection policy with rate $\boldr$. For the ease of exposition we postpone that discussion to \cref{sec:defining-h} in favor of first presenting our proposed algorithm.

\feast ({\bf F}ederated {\bf A}veraging {\bf A}ided by an {\bf A}daptive {\bf S}ampling {\bf T}echnique), is presented as Algorithm~\ref{alg:gs_fedavg}. Formally, \feast aims to find a configuration-dependent client sampling strategy $\bm{f_{\text{\vargrad}}}$ such that its long-term client participation rate $\boldr_{\text{\vargrad}}$ approximates the optimal achievable strategy $\boldr^* \in \argmin_{\boldr \in \calR} H(\boldr)$. 
To accomplish this, \feast first initializes $\boldr(0)$ arbitrarily (line 2). At each round $t$, with $\bC_t= C_t$, selecting set $S\in C_t$ implies a contribution to the participation rate of $1$ for every client $k \in S$. Whether or not selecting set $S$ brings $r(t)$ closer to $\boldr^*$ can be computed by estimating the marginal utility of $S$ using the gradient of $H(\boldr)$. Thus, we select $S_t$ (line 5) as 
\begin{equation}
    \label{eq:gradient_sched}
    S_t \in\arg\max_{S \in C_t } -\nabla H(\boldr(t)) \cdot \mathds{1}_S.
\end{equation}
In general, (\ref{eq:gradient_sched}) is a combinatorial optimization problem; in the federated learning systems with a time-varying bound on the number of clients $K_t$ that can be selected ($K_t \ge 0$), (\ref{eq:gradient_sched}) reduces to the discrete optimization problem of greedily selecting $K_t$ available clients with the largest entries of $-\nabla H (\boldr(t))$. Optimality of the greedy approach follows because the objective is an additive set function. 


\begin{figure*}[h]
       \centering
    \includegraphics[width=\textwidth]{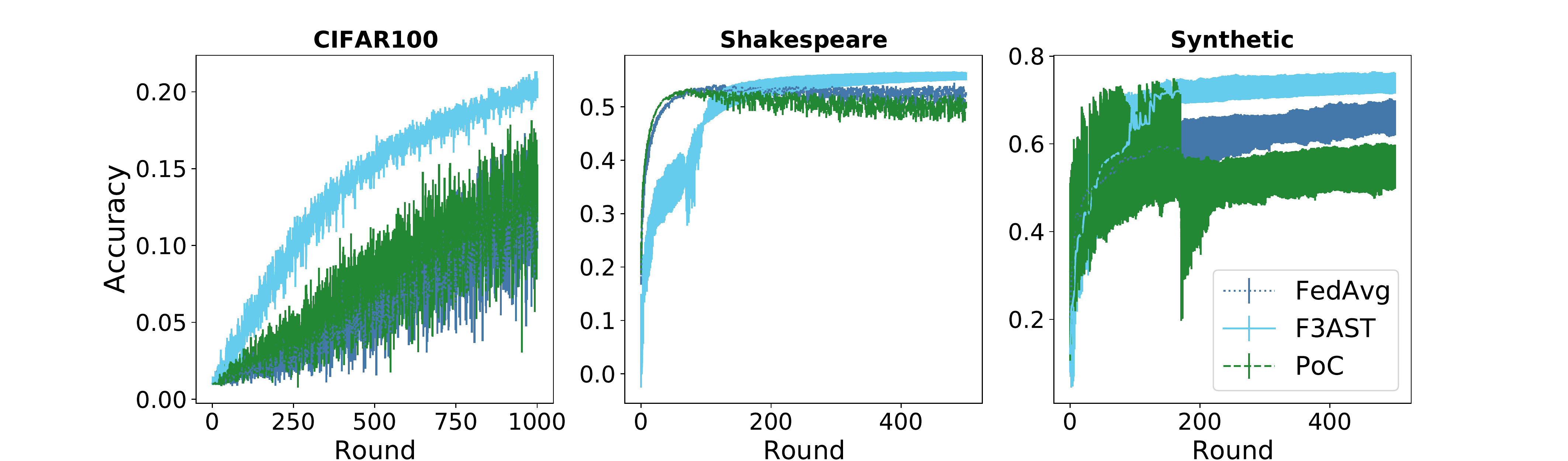}
    \caption{Test per-sample accuracy (averaged over three runs) for different client sampling and aggregation schemes under {\em HomeDevice} availability model. We observe that \vargrad consistently outperforms \fedavg and \poc. Further, \vargrad stabilizes while \fedavg and \poc are unable to adapt to the time-varying environment. }
    \label{fig:stable_convergence}
\end{figure*}

Next, the rate is updated to reflect the selection made in the latest iteration of the sampling scheme. This is done by forming an exponentially smoothed average of the past sampling rates (line 6),
\begin{equation}
\label{eq:update_service_rate}
\boldr(t+1) = (1-\beta)\boldr(t) + \beta \mathds{1}_S,
\end{equation}
where $\beta>0$ is a fixed small parameter. After having selected clients $S_t$, the server broadcasts the current model and the selected clients perform local updates. Finally, the server uses $\boldr(t)$ to produce an unbiased global model $\avgw^{t+1}$ with estimator $\Delta^{t+1}$ (lines 9-10).

\begin{algorithm}
\begin{algorithmic}[1]
\REQUIRE Server parameters: learning rate schedule $\{ \eta_t \}_{t=1}^T$, the number of global rounds $T$, the number of clients per round $K_t$, the number of client local updates $E$
\ENSURE Global model  $\avgw_{T}$
\STATE initialize $\avgw_0 \in \mathbb{R}^p$ arbitrarily,
initialize $\boldr(0)$ arbitrarily
\FOR{$t=1 \to T$}
    \STATE $C_t \quad \gets $ feasible client sets at time $t$ 
    \STATE $S_t \in \argmin_{S \in C_t} \nabla H(\boldr_t)\mathds{1}_S$
    \STATE $\boldr(t) = (1-\beta)\boldr(t-1) + \beta \mathds{1}_S$
    \FOR{Clients $k \in S_t$, in parallel}
        \STATE $\boldv_{k}^{t+1} \gets \textsc{ClientOpt}( \avgw^t, E \text{ steps}, \eta_t )$
    \ENDFOR
    \STATE $\Delta^{t+1} \gets \sum_{k \in S_t} \frac{p_k}{r_k(t)}\boldv_k^{t+1}$
    \STATE $\avgw^{t+1} \gets \textsc{ServerOpt}(\avgw^{t},\Delta^{t+1})$
\ENDFOR
\end{algorithmic}
\caption{\feast: Federated Averaging Aided by an Adaptive Sampling Technique}
\label{alg:gs_fedavg}
\end{algorithm}

\paragraph{Beyond \fedavg\!.} \feast modifies two crucial steps in \fedavg\!: client sampling and model updates aggregation. This makes it suitable to work in combination with other FL algorithms like \textsc{Scaffold} \cite{karimireddy2020scaffold}, AFL \cite{mohri2019agnostic}, \textsc{FedProx} \cite{li2018federated}, \textsc{FedDyn} \cite{acar2021federated}, and more generally \textsc{FedOpt} \cite{reddi2020adaptive}. These methods' theoretical guarantees are provided under the ``all clients are available" assumption, implying that they assume an unrealistic fixed sampling policy which introduces bias to the model. Our proof extends to those settings by modifying accordingly sampling and aggregation to the asymptotically learned $\boldr$ as long as the $\ell_2$-norms of clients' model updates are uniformly bounded -- an assumption already made by the above methods. 

\subsubsection{Asymptotic optimality}

Below we show that as $\beta \downarrow 0$, the selection policy rate converges to the value that optimizes $H(\boldr)$ and, therefore, reduces the model variance. To this end, consider the discrete time Markov process $S^{\beta}(t) = (\boldr^{\beta}(t),\bC_t)$ indexed by the value of $\beta$ defined in Eq. (\ref{eq:update_service_rate}), with $\beta \downarrow 0$ along sequence $\calB = \{ \beta_j\}_{j\in \mathbb{N}}$. $S^{\beta}(0)$ and the probability law of $\bC_t$ describing the availability model are fixed for all $\beta \in \calB$. The speed of convergence is discussed in the appendix. 

\begin{theorem}
\label{thm:sampling_rate_convergence}
Let $\boldr^{\beta}(t)$ be the rate determined by \cref{alg:gs_fedavg}. Let $V \subset \mathbb{R}_{+}^N$ be a bounded set, $\epsilon>0$, and let $\boldr^*$ denote the minimizer of the variance function $H(\boldr)$ over $\calR$. Then for $T>0$, depending on $\epsilon$ and $V$,
\begin{equation*}
   \lim_{\beta \downarrow 0}\quad \sup_{\boldr^{\beta}(0) \in V, t>T/\beta}P [\| \boldr^{\beta}(t) - \boldr^*\|>\epsilon]=0.
\end{equation*}
\end{theorem}


\subsubsection{Bounding the global model variance}
\label{sec:defining-h}

We start by analyzing a fixed arbitrary stochastic policy $\bm{f}^{\boldr}$ achieving rate $\boldr$ and use the result to demonstrate that $H($\boldr$)$ reflects the model variance induced by $\bm{f}^{\boldr}$ (for more details, please see appendix).
Let us introduce several assumptions regarding clients' loss functions
$F_k(\boldw) = \Expec{\xi\sim\calD_k}{f_k(\boldw;\xi)}$, $1 \le k \le N$; these assumptions are commonly encountered in the federated learning literature \cite{li2019convergence, cho2020client}.

\begin{assumption}{[Smoothness and strong convexity]} \label{assum:smooth} $F_1,...,F_N$ are $L-$smooth and $\mu-$strongly convex functions, meaning that
for all $v$ and $w$, $F_k(\boldv) \leq F_k(\boldw) +\nabla F_k(\boldw)^T(\boldv - \boldw) +\frac{L}{2}\|\boldv - \boldw \|_2^2$ and $F_k(\boldv) \geq F_k(\boldw) +\nabla F_k(\boldw)^T(\boldv - \boldw) +\frac{\mu}{2}\|\boldv - \boldw \|_2^2$, respectively.
\end{assumption}

\begin{assumption}{[Bounded variance]}
\label{assum:bound_variance}
Let $\xi$ be a data point that client $k$ samples from distribution $\calD_k$. Then $\mathbb{E}_{\xi\sim \calD_k}[\|\nabla f_k(\boldw, \xi)-\nabla F_k(\boldw) \|_2^2] \leq \sigma_k^2$, $k=1,...,N$
\end{assumption}

\begin{assumption}{[Bounded stochastic gradients]}
\label{assum:bound_gradient_norm}
The expected norm of the stochastic gradients of $f_k$ is uniformly bounded, i.e., $\mathbb{E}_{\xi \sim \calD_k}[\|\nabla f_k(\boldw,\xi) \|_2^2] \leq G^2$, $k=1,...,N$.
\end{assumption}
  
Assumption~\ref{assum:smooth} holds in a number of scenarios of interest, including $\ell_2$-regularized linear and logistic regression, and classification with softmax function. Assumptions~\ref{assum:bound_variance} and \ref{assum:bound_gradient_norm} are common in state-of-the-art distributed learning literature \cite{zhang2013communication,stich2018local, stich2018sparsified, taheri2020quantized}. Note that while we rely on the above assumptions when analyzing the performance, experimental results demonstrate that our sampling techniques work very well in more general settings involving highly non-linear models such as convolutional and recurrent neural networks trained on realistic datasets. 




The following lemma introduces and analyzes $\sigma^2_t(\bm{f}^{\boldr})$, the client sampling variance under sampling policy $\bm{f}^{\boldr}$.

\begin{lemma}
\label{lemma:bound-variance}
Suppose Assumptions~\ref{assum:constraint_process}-\ref{assum:bound_gradient_norm} hold. Let $\boldr \in \calR$ be an achievable long-term client participation rate under the system configuration determined by distribution $\pi$, and $\bm{f}^{\boldr}$ denote a static configuration-dependent selection policy achieving  $\boldr$. Define the client sampling variance $\sigma^2_t(\bm{f}^{\boldr}): = \frac{1}{\eta_t^2}\Expec{S_t}{\|\Delta^{t}-\avgv^t \|^2}$ where $\avgv^t = \sum_{k=1}^Np_k\boldv^k_{t}$ is the update at time $t$ with full client participation. Then
\begin{equation}
    \sigma_t^2(\bm{f}^{\boldr}) \leq 4E^2G^2  (\sum_{k=1}^N\frac{p_k}{r_k} - 1).
\end{equation}
 Furthermore, if client availabilities are uncorrelated or negatively correlated, then there exists a policy $\bm{f}^r$ such that 
\begin{equation}
    \sigma_t^2(\bm{f}^{\boldr}) \leq 4E^2G^2 \left( \sum_{k=1}^N\frac{p_k^2}{r_k} + \sum_{k=1}^N p_k^2 \right).
\end{equation}
\end{lemma}

\begin{theorem}
\label{thm:importance_convergence}
Instate the settings of \cref{lemma:bound-variance}.
Let $\boldw^*$ denote the solution to the optimization problem (\ref{eq:objective_function}), and $L,\mu=O(1)$.  Define $\gamma = \max \{8\frac{L}{\mu}, E\}$, and assume learning rate $\eta_t = \frac{2}{\mu(\gamma+t)}$. 
Then by setting \textsc{ClientOpt} to SGD and \textsc{ServerOpt}$(\avgw^t, \Delta^{t+1}) = \avgw^t+\Delta^{t+1}$, the model $\avgw^T$ produced with policy $\bm{f}^{\boldr}$ after $T$ steps satisfies
\begin{align}
    \label{eq:convergence_bound}
 \nonumber   \mathbb{E}[F(\avgw^T)] &- F^* =O \left( \frac{1}{TE+\gamma} \left(  \|\boldw_1 - \boldw^* \|^2  \right. \right. \\
 & +\sum^N_{k=1}p_k^2\sigma_k^2 + \left.\left. 6\Gamma + 8(E-1)^2G^2 + \sigma_T^2(\bm{f}^{\boldr}) \right)\right),
\end{align}
where $\Gamma = F^* - \sum_{k=1}^Np_kF_k^*$ denotes the local-global objective gap\footnote{Local-global objective gap quantifies data heterogeneity: for i.i.d. data, $\Gamma \to 0$ as the number of samples grows, while a large $\Gamma$ indicates a high degree of heterogeneity \cite{li2019convergence, cho2020client}.}, $F^*$ and $F_k^*$ are the minimum values of $F$ and $F_k$, respectively, and  $\sigma_T(\boldr)$ (\cref{lemma:bound-variance}) captures the variance induced by client sampling.
\end{theorem}

\begin{remark}
\normalfont Proof of \cref{thm:importance_convergence} follows the line of argument similar to that in the analysis of FL algorithms convergence \cite{stich2018local, li2019convergence, reddi2020adaptive}. The first term in the parenthesis in eq.~(\ref{eq:convergence_bound}) captures the effect of initialization, while the second term reflects the variance of stochastic gradients. The remaining terms are tied to the inherent challenges of data heterogeneity in FL. 

 In summary, our presented technical contributions include: (i) developing an arbitrary selection policy with feasible long-term client participation rate $\boldr$ (meanwhile, the result of \citet{li2019convergence} applies only to two specific - potentially unfeasible in real settings - sampling schemes); (ii) showing the global update is unbiased under this selection policy and aggregation step; and (iii) computing the incurred variance captured by the last term $\sigma_T^2(\boldr)$ and bounded in \cref{lemma:bound-variance}.  
\end{remark}
\vspace{-0.1in}
\paragraph{From static to dynamic policies.} Aiming to circumvent the requirement for having access to (generally unavailable) system configuration information that applies to static selection policies, we proceed by combining expressions (6) and (7) (ignoring constant terms) to define
\begin{equation*}
    H(\boldr):= \begin{cases} 
    \vspace{-0.05in}
            \sum_{k=1}^N\frac{p_k}{r_k} & \text{client availability is }\\
            & \text{positively correlated,} \\
            \sum_{k=1}^N\frac{p^2_k}{r_k} & \text{ otherwise.} \\
                \end{cases}
\end{equation*}
Clearly, minimizing $H(\boldr)$ over $\boldr$ reduces the upper bound on the model variance
stated in \cref{lemma:bound-variance}.
It follows from \cref{eq:client-rate} that $r_k \leq 1$, and thus it can readily be shown that $r_k = 1$ for all $k$ minimizes $H(\boldr)$ over $[0,1]^N$. However, it is possible that for a given configuration of the system this long-term client participation rate is not achievable, i.e., $\bm{1} \notin \calR$ (the set of feasible $\boldr$'s defined in \cref{lemma:constraint_region}). Recalling the example in \cref{sec:intro},
 minimization of $H(\boldr)$ over $\calR$ is also difficult because it requires knowledge of the achievable long-term client participation rate region $\calR$, determined by the distribution $\pi$ defined in Assumption~\ref{assum:constraint_process}. Finally, even if $\calR$ is known, the resulting policy $\bm{f}^{\boldr}$ will likely be client and set dependent, thus rendering the problem challenging due to an exponential number of variables and unknown parameters. 
These are precisely the obstacles that \feast overcomes by learning a selection policy which is asymptotically optimal in terms of minimizing $H(\boldr)$, and guaranteeing convergence to a long-term participation rate minimizing $H(\boldr)$ over $\calR$. 

\begin{table*}[h]
\caption{Test sample accuracy of all methods, and relative improvements of \vargrad over \fedavg and \vargrad+ Adam over \fedadam\!\!, for different availability models (columns) on CIFAR100 (1000 rounds) and Shakespeare (500 rounds).}
\label{tab:accuracy}
\centering
\begin{tabular}{@{}lllllll@{}}
\toprule
~ & ~ & \multicolumn{5}{c}{Availability models} \\ 
~ & ~      &    Always       &      Scarce     &      HomeDevice     &      Uneven     &      SmartPhones \\
\midrule
 \multirow{5}{*}{CIFAR100} & \textcolor{cb-blue}{\fedavg}     &             0.141                              &     0.096     &     0.111     &     0.072     &     0.142 \\
& \feast      &             0.198 \textcolor{cb-blue}{(+40\%)}    &     0.201   \textcolor{cb-blue}{(+109\%)}  &     0.208  \textcolor{cb-blue}{(+87\%)}   &     0.206 \textcolor{cb-blue}{(+186\%)}    &     0.201 \textcolor{cb-blue}{(+42\%)}\\
\clineThicknessColor{2-7}{0.3pt}{lightgray}
& \textcolor{cb-lilac}{\fedadam}    &             0.262                              &     0.288     &     0.302     &     \textbf{0.282}     &     0.298 \\
& \feast+ Adam &            \textbf{0.271}  \textcolor{cb-lilac}{(+3\%)}          &     \textbf{0.308}   \textcolor{cb-lilac}{(+7\%)}  &     \textbf{0.324}   \textcolor{cb-lilac}{(+7\%)}  &     \textbf{0.281} \textcolor{gray}{(0\%)}    &     \textbf{0.320} \textcolor{cb-lilac}{(+7\%)} \\\clineThicknessColor{2-7}{0.3pt}{lightgray}
& \poc &             0.101     &     0.115     &     0.139     &     0.111     &     0.069 \\ \midrule
 \multirow{5}{*}{Shakespeare}& \textcolor{cb-blue}{\fedavg}  &             0.54     &     0.549     &     0.522     &     0.540     &     0.538 \\
& \feast &             \textbf{0.569} \textcolor{cb-blue}{(+5\%)}     &     \textbf{0.555}  \textcolor{cb-blue}{(+1\%)}    &     \textbf{0.568} \textcolor{cb-blue}{(+9\%)}     &     0.556 \textcolor{cb-blue}{(+3\%)}     &     0.570 \textcolor{cb-blue}{(+6\%)}  \\\clineThicknessColor{2-7}{0.3pt}{lightgray}
& \textcolor{cb-lilac}{\fedadam} &             0.536     &     0.541     &     0.520     &     \textbf{0.557}     &     0.520 \\
& \feast + Adam &             0.549  \textcolor{cb-lilac}{(+2\%)}    &     0.551   \textcolor{cb-lilac}{(+2\%)}   &     0.566  \textcolor{cb-lilac}{(+9\%)}     &     \textbf{0.557} \textcolor{gray}{(0\%)}     &     0.551 \textcolor{cb-lilac}{(+6\%)} \\\clineThicknessColor{2-7}{0.3pt}{lightgray}
& \poc &             0.554     &    \textbf{0.555}     &     0.496     &     0.555     &     0.535 \\
\bottomrule
\end{tabular}
\end{table*}
\begin{table*}[h]
\caption{Sample loss of all methods, and relative improvements of \vargrad over \fedavg and \vargrad+ Adam over \fedadam\!\!, for different availability models (columns) on CIFAR100 (1000 rounds) and Shakespeare (500 rounds).}
\label{tab:loss}
\centering
\begin{tabular}{@{}lllllll@{}}
\toprule
~ & ~ & \multicolumn{5}{c}{Availability model} \\ 
~ & ~      &    Always       &      Scarce     &      HomeDevice     &      Uneven     &      SmartPhones \\
\midrule
 \multirow{5}{*}{CIFAR100} & \textcolor{cb-blue}{\fedavg}  &             4.42     &     4.77     &     4.74     &     5.29     &     4.28 \\
 & \feast &             4.20 \textcolor{cb-blue}{(-5\%)}     &     4.18  \textcolor{cb-blue}{(-12\%)}   &     4.14 \textcolor{cb-blue}{(-13\%)}    &     4.17 \textcolor{cb-blue}{(-21\%)}     &     4.15 \textcolor{cb-blue}{(-3\%)} \\ \clineThicknessColor{2-7}{0.3pt}{lightgray}
& \textcolor{cb-lilac}{\fedadam} &             3.74     &     3.65     &     3.50     &     3.67     &     3.55 \\
& \feast+ Adam &             \textbf{3.69}  \textcolor{cb-lilac}{(-1\%)}   &     \textbf{3.61}  \textcolor{cb-lilac}{(-1\%)}   &     \textbf{3.45}   \textcolor{cb-lilac}{(-1\%)}  &     \textbf{3.66}   \textcolor{cb-lilac}{(-0.5\%)}  &     \textbf{3.46} \textcolor{cb-lilac}{(-2\%)} \\ \clineThicknessColor{2-7}{0.3pt}{lightgray}
& \poc &             4.90     &     4.84     &     4.63     &     4.86     &     5.42 \\
\midrule
 \multirow{5}{*}{Shakespeare} & \textcolor{cb-blue}{\fedavg}  &             1.24     &     1.15     &     1.36     &     1.17     &     1.22 \\
 & \feast &             \textbf{1.10} \textcolor{cb-blue}{(-11\%)}    &     \textbf{1.13} \textcolor{cb-blue}{(-1\%)}    &     \textbf{1.10}  \textcolor{cb-blue}{(-19\%)}   &    1.13 \textcolor{cb-blue}{(-3\%)}  &     \textbf{1.10} \textcolor{cb-blue}{(-10\%)}\\ \clineThicknessColor{2-7}{0.3pt}{lightgray}
 & \textcolor{cb-lilac}{\fedadam} &             1.27     &     1.23     &     1.40     &     1.12     &     1.36 \\
 & \feast+ Adam &             1.18 \textcolor{cb-lilac}{(-8\%)}    &     1.19  \textcolor{cb-lilac}{(-3\%)}   &     1.11  \textcolor{cb-lilac}{(-21\%)}   &    \textbf{ 1.11 }  \textcolor{cb-lilac}{(-1\%)}  &     1.18 \textcolor{cb-lilac}{(-13\%)} \\ \clineThicknessColor{2-7}{0.3pt}{lightgray}
 & \poc &             1.13     &     \textbf{1.13}     &     1.74     &     1.13     &     1.24 \\
\bottomrule
\end{tabular}
\end{table*}
\vspace{-0.1in}
\section{Experiments}
\label{sec:experiments}
\subsection{Datasets and models}

\label{sec:experiments_description}
We test our model on three well-known federated datasets. First, a synthetic heterogenous dataset Synthetic(1,1) for softmax regression, introduced in \cite{shamir2014communication} and widely used in the FL community \cite{cho2020client,li2018federated, li2019convergence}. Second, a recurrent neural network with 1M parameters for the next character prediction task on the Shakespeare dataset \citep{mcmahan17a}, a language modelling dataset with 725 clients, each one a different speaking role in each play from the collective works of William Shakespeare.  Third, CIFAR100 with the partition introduced in \cite{reddi2020adaptive}, utilizing Latent Dirichlet Allocation in order to generate a realistic heterogenous distribution. We train ResNet-18, replacing batch with group normalization, a modification that has shown improvements in federated settings \citep{hsieh2020non}. Our code is available on Github\footnote{\url{https://github.com/mriberodiaz/f3ast}}

\vspace{-0.1in} 
\paragraph{Availability models.}
We perform tests on five realistic availability models described below. To our knowledge, there exist no public databases with real availability patterns; Smartphone's model \cite{bonawitz2019towards} is inspired by realistic data. All models are motivated by practical federated learning systems:


\begin{enumerate}
\item
\textit{Always}: Baseline model, clients are always available. 
\vspace{-0.1in}
\item \textit{Scarce}: Independent and homogeneous availability across clients and time with probability $q=0.2$.
\vspace{-0.1in}
\item \textit{Home-devices}: Independent availability across clients and time with probability $q_k = T_k/B$, where $T_k\sim \mbox{\rm lognormal}$ and $B = \max_k T_k$.
\vspace{-0.1in}
\item \textit{Smartphones}: Sine-modulated {\em Home-devices} model, $q_{k,t}=f_tq_k$, where $q_k$ is defined in the \textit{Home-devices} model and $f_t$ denotes a sinusoidal time-dependent availability (see \cite{bonawitz2019towards}).
\vspace{-0.1in}
\item \textit{Uneven:} Each client's availability is inversely proportional to its dataset size, $q_k \propto 1/p_k$. 
\end{enumerate}

\vspace{-0.1in}

We split each client's dataset into training and validation sets. We assume the distribution $\calP$ over users is determined by the fraction of data they possess. In the following, we first fix the communication constraint to select $M=10$ clients in each round and compare different methods across availability models. We then proceed by exploring the effect of varying $M$. Further details on the data sets, model architectures, availability modeling, and hyperparameters can be found in Appendix~\ref{sec:app:architectures}. We implement our models using the Tensorflow-Federated API \citep{tff}. 

\vspace{-0.1in}

\paragraph{Baselines.}

First, we compare our algorithm with two availability-agnostic methods: (i) \fedavg\!, a standard baseline, and (ii) \fedadam\!\!, which achieves state-of the-art performance in the considered benchmark tasks \cite{reddi2020adaptive}. Both methods sample available clients with normalized probabilities $p_k$, but \fedavg uses SGD as the server optimizer while \fedadam uses Adam. For a fair comparison, we implement both methods and compare with their availability-aware versions wherein we incorporate our proposed sampling and aggregation step.

Second, we test against a state-of-the-art algorithm, Power-of-Choice (\poc\!) \citep{cho2020client}, a method that, although agnostic to the availability model, can work in conjuction with client unavailability. In \poc, the server at time $t$ samples $d$ clients from the available set $C_t$ without replacement, choosing client $k$ with probability $p_k$. The $d$ clients receive the current model $\boldw_t$ and inform the server about their current loss $F_k(\avgw^t)$. The server then selects for training the top $M$ clients with the highest loss.

We do not compare our method with stateful techniques (\textsc{Scaffold}, \textsc{FedDyn}) since they are not applicable in the cross-device federated settings \cite{kairouz2019advances}. 

To evaluate performance of the algorithms, we compute the loss and accuracy using per-test-sample averages (the average is taken over individual data points). Per-user average results can be found in the appendix.
\vspace{-0.1in}
\subsection{Numerical results.}
\label{sec:experiments-availability-models}

We first show the convergence of \vargrad on the three datasets with {\em Home-devices} availability model, a setting that fits Assumption~\ref{assum:constraint_process} and is realistic in FL (note that the synthetic dataset satisfies all assumptions from Section~3). Corroborating expectations of the impact of the de-biasing step introduced by \feast\!\!, \cref{fig:stable_convergence} shows that \vargrad achieves higher accuracy than \fedavg and \poc on all datasets (the corresponding loss plot can be found in \cref{sec:app_experiments}).  Moreover, after the first 100 iterations, \feast stabilizes on Shakespeare and Synthetic datasets, and follows a more stable learning trajectory, illustrating its variance reduction advantage, unlike the two baselines that have high variability due to time-varying client availability. The sharp drop in Shakespeare is caused by sampling a client misaligned due to the heterogenous nature of the data; this has been reported in \cite{charles2021large}. We show an average over three runs in \cref{sec:app_experiments}.
We observe a similar behaviour on CIFAR100: \vargrad achieves almost a 200\% improvement in the average accuracy over the last 100 rounds, and has a much more stable behaviour. The stagnation of \fedavg and \poc at higher loss models confirms that naive averaging introduces bias to the model, hindering convergence.

\cref{tab:accuracy} shows the final accuracy of algorithms under diverse availability models defined in Section~\ref{sec:experiments_description} (1000 rounds on CIFAR100 and 500 rounds on Shakespeare \footnote{Higher accuracy values on CIFAR100 could be obtained by running experiments for 10,000+ rounds}). \vargrad effectively improves the accuracy of \fedavg over both datasets and for all availability models. It also improves \fedadam for all but the {\em Uneven} model where the accuracy remains the same although the loss value is lower (\cref{tab:loss}). \vargrad is particularly successful in difficult settings, e.g. {\em Scarce} and {\em Uneven}, where a small number of clients is available for training and where the client availability is inversely proportional to the amount of data clients hold -- there, \vargrad is able to maintain performance similar to the setting where all clients are available. Meanwhile, the performance of both \fedavg and \poc deteriorates. \fedadam is able to maintain performance in the {\em Uneven} model where momentum may help with biased updates, but deteriorates in all other settings. 

\vspace{-0.1in}
\section{Conclusion}

We presented \feast, an algorithm for learning in federated systems that operate under communication constraints and service intermittently available clients. We demonstrated that the algorithm achieves accuracy superior to state-of-the-art federated learning techniques, and exhibits resilience in challenging system settings. Future work includes studies of the setting where the clients are grouped in clusters/classes, and exploring a wider range of communication constraints. 

\section*{Acknowledgements}
This material is based upon work supported by the National Science Foundation under grant no. 2148224 and is supported in part by funds from OUSD R\&E, NIST, and industry partners as specified in the Resilient \& Intelligent NextG Systems (RINGS) program.

\newpage

\bibliography{references}
\bibliographystyle{plainnat}

\newpage
\appendix
\onecolumn

\section*{\center{Supplementary Material for \textit{``Federated Learning Under Time-Varying Communication Constraints and Intermittent Client Availability"}}} 

\vspace{0.5cm}

This document is organized as follows. Section~A introduces notation and states useful definitions. Section~B presents discussion of \feast extensions and problem variations.
Section~C provides proofs of theorems and lemmas that, for the sake of brevity, have been omitted from the main document. Finally, details of the experiments are given in Section~D.

\section{Notation and Definitions}
\label{sec:app:notation}

For clarity, frequently used symbols are summarized in Table~1 below.

\begin{table*}[ht]
\centering
\begin{tabular}{ll}
\hline
Symbol                                             & Definition                                                         \\ \hline
$\calU$                                                  & Set of all users \\
$N = |\calU|$                                                    & Total number of clients\\
$K_t$                                                    & Number of clients sampled at round $t$ \\
$T$                                                    & Total number of rounds        \\
$m_t$                                           & Number of available clients at time $t$ \\
$(\bA_t)_t$                                           & availability stochastic process \\
$(\bC_t)_t$                                           & feasible client sets stochastic process\\
$S_t$                                           & Set of clients participating at round $t$ \\
$\pi(\cdot)$                                           & Stationary distribution of availability and communication constraint process. \\
$\boldr$                                           & Client sampling rate \\
$\bm{f}^{\boldr}$                               & Configuration-dependent client sampling policy \\
$\Delta^{t+1}$                                   & Pseudo-gradient for server optimizer: aggregates updates of participating clients at round $t$    \\
$\avgw^t$                                       & Global model at beginning of round $t$                                          \\
$\boldw_k^{t+1}$                                 & Local model at the end of round $t$ at client $k$\\ 
$\boldv_k^{t+1}$                                 & Local update at the end of round $t$ at client $k$\\ 
$\avgv^{t+1} = \Expec{k\sim \calP}{\boldv^k_{t+1}} = \sum_{k=1}^Np_k\boldv^k_{t+1}$                                 & Expected global update at the end of round $t$ under desired distripution $\calP$\\ 
$\avgz^{t+1} = \avgw^t+\avgv^{t+1}  $                             & Desired global model at the end of round $t$ \\ 
 \bottomrule
\end{tabular}
\caption{Frequently used symbols.}
\label{tab:notation}
\end{table*}


The generalized \fedavg assumes the server sends to clients in $S_t$ at time $t$ an initial model $\avgw^t$. For $t=0,...,T$; the clients locally initialize $\boldw_k^{(t,0)} \gets \avgw^{t-1}$ and take $E$ steps of SGD producing the sequence $(\boldw_k^{(t,i)})_{i=0}^E$. Formally, let $\xi^{(t,i)}_k$ be the mini-batch for client $k$ at time $i$ in round $t$; for each client $k$, we can then define the local model $\boldw^{(t,i)}_k$ and local update $\boldv_k^{t}$ as

\begin{equation*}
    \boldw^{(t+1,i)}_k =  \begin{cases}
     \avgw^{t} & i=0 \\
    \boldw^{(t+1,i-1)}_k-\eta_{t+1,i} \nabla F_k(\boldw^{(t+1,i-1)}_k, \xi^{(t+1,i)}_k), & i=1,...,E,
    \end{cases} \hspace{1cm}  \boldv^{t+1}_k = \boldw_k^{(t+1,E)} - \avgw^{t}.
\end{equation*}

Here $\boldw_{k}^{(t,i)}$ tracks local models of client $k$ at round $t$ and iteration $i$, and $\boldv_k^{t}$ is the local update of client $k$ at the end of round $t$. Following distributed optimization standard techniques \cite{stich2018local, li2019convergence}, we define the sequences

\begin{equation}
    \label{eq:average_model}
    \overline{\boldv}^{t+1} = \sum_{k=1}^Np_k\boldv^{t+1}_k, \hspace{0.5cm}  \Delta^{t+1}= \sum_{i \in S_t} \frac{p_i}{r_i}\boldv_k^{t+1}, \hspace{0.5cm}  \avgw^{t+1} = \avgw^{t}+\Delta^{t+1}, \hspace{0.5cm}  \avgz^{t+1} = \avgw^t+\avgv^{t+1}.
\end{equation}


\section{\feast Extensions and Problem Variations.}

\paragraph{Rapid mixing time. } Convergence rate of $\boldr(t) \to \boldr^*$ depends on the properties of the Markov chain specified by the availability process. Concretely, we have the following known theorem (see, e.g., \cite{LP17} for a proof and details).

\begin{theorem}{[Convergence Theorem]}
\label{thm:mixing-time}
Let $P$ be the transition matrix of a system configuration satisfying \cref{assum:constraint_process} with stationary distribution $\pi$. Then there exists a constant $\alpha \in (0,1)$ and $C>0$ such that 

$$\max_{C \in \calC} \|P^t(C,\cdot) - \pi \|_{TV}\leq C\alpha^t.$$

\end{theorem}

The above result shows that in practice one needs $t_0 \geq \frac{\log \epsilon}{\log \alpha}$ iterations to achieve a stationary rate $\boldr$ up to an error $\epsilon$ to the stationary distribution. Further, even if the rate is not constant during the early iterations, this result demonstrates that after a burn-in the rate will stabilize and the asymptotic convergence rate of $O(1/TE)$ will not be affected.

\paragraph{Working with clusters of clients. } Note that in some settings the number of clients is rather large and thus tracking $p_k$ could be difficult. In such scenarios, it may be meaningful to interpret $p_k$, $k=1,...,N$, as a partition of data over clusters of clients, and the objective function in (\ref{eq:objective_function}) as a weighted average over the clusters. Tracking $p_k$ can then be cast as the estimation of time-varying class sizes, which is feasible through efficient protocols \cite{alouf2002optimal}.

\paragraph{More general sampling constraints. } 
Our model is quite general in that it can capture different client sampling requirements, e.g., forcing the selection of as many clients as available up to $M$ or other minimal sampling requirements on groups of users. 

\paragraph{Societal impact.} Our work is motivated by model fairness in the sense that clients are represented in the global model proportionally to their dataset size. Notice that our framework adapts to different fairness metrics by replacing distribution $\calP$. This is out of the scope of our paper, and we leave it for future research. 

\paragraph{Fixed-policy \feast\!\!.} For the sake of completeness, we here formalize as \cref{alg:importance_fedavg} the procedure that can determine the fixed policy analyzed in Section~3.2.

\begin{algorithm}
\begin{algorithmic}[1]
\REQUIRE Server parameters: learning rate schedule $\{\eta_{(t,i)}\}_{_{(t,i)}} $, the number of global rounds $T$, targeted long-term rate $\boldr \in \calR$ and static configuration dependent sampling policy $\bm{f}^{\boldr}$ that achieves $\boldr$
\ENSURE Global model  $\avgw^{T}$
\STATE initialize $\avgw^0$
\FOR{$t=1 \to T$}
    \STATE$C_t \quad \gets $ available collection of possible client subsets at time $t$
    \STATE Use policy $\bm{f^r}$ to select $S_t\in C_t$
    \FOR{Clients $k \in S_t$, in parallel}
        \STATE $\boldv_{k}^{t+1} \gets \textsc{ClientOpt}\left( \avgw^t, E \text{ steps}, \{\eta_{(t,i)}\}_{_{(t,i)}} \right)$
    \ENDFOR
    \STATE $\Delta^{t+1}= \sum_{k \in S_t} \frac{p_k}{r_k}\boldv^{t+1}_k$
    \STATE $\avgw^{t+1} \gets \textsc{ServerOpt}(\avgw^{t},\Delta^{t+1})$
\ENDFOR
\end{algorithmic}
\caption{Fixed-policy \feast\!\!. }
\label{alg:importance_fedavg}
\end{algorithm}

\section{Proofs}
\label{sec:app:proofs}

 This section provides proofs of the lemmas and theorems omitted from the main document.

\subsection{Proof of Theorem~\ref{thm:importance_convergence}}

Proof of \cref{thm:importance_convergence} follows standard optimization proofs \cite{stich2018local, li2019convergence, reddi2020adaptive}. Our technical contribution comes from using a sampling policy with arbitrary sampling rate $\boldr$, showing global update $\Delta^{t+1}$ is unbiased (\cref{lemma:unbiased}), and computing the incurred variance of such sampling policy and aggregation step (\cref{lemma:bound-variance}). This last step is of particular interest and challenging due to the unknown system configuration, and the importance sampling multiplicative terms. 

\begin{lemma}[Unbiased update]
\label{lemma:unbiased}
Suppose Assumption~\ref{assum:constraint_process} holds. Let $\boldr \in \calR$ be an achievable sampling rate, and $\bm{f}^r$ be a state-dependent static policy achieving rate $\boldr$. Fix $\avgw^{t}\in \mathbb{R}^p$. Let $(\boldv_k^{t+1})_{k=1}^N$ denote updates of clients starting from model $\avgw^{t}$. Let $\avgv^{t+1} =\sum_{k=1}^Np_k\boldv_k^{t+1}$, let $S$ be the client set selected by policy $\bm{f}^r$ at time $t$ and $\Delta^{t+1} = \sum_{k \in S}\frac{p_k}{r_k}\boldv_k^{t+1}$. Then $\mathbb{E}_S[\Delta^{t+1}] = \avgv^{t+1}$.
\end{lemma}

\begin{proof}
Recall that at any given time $t$ we pick $S \in C_t$ for some $C_t \in \calC$ according to $f^r_{C_tS}$. 
Using the definition of $\pi$ and $\bm{f}^r$, 

\begin{align}
 \Expec{S}{\Delta^{t+1}} &=  \Expec{}{ \Expec{}{\sum_{k \in S} \frac{p_k}{r_k}\boldv_k^{t+1} \big| C}}
 = \sum_{C \in \calC}\pi(C) \sum_{S \in C} f^r_{CS} \sum_{k \in S}\frac{p_k}{r_k}\boldv_k^{t+1}\\
 &=\sum_{C\in\calC}\pi(C)\sum_{S \in C} f^r_{CS} \sum_{k=1}^N \frac{p_k}{r_k}\boldv_k^{t+1}\cdot \mathds{1}_{\{k\in S\}},
\end{align}
where the last step replaces the sum over $S$ by the sum over all clients but adding the indicator function over $S$. Reorganizing,
\begin{align}
   \Expec{S}{\Delta^{t+1}}
   &= \sum_{k=1}^N \frac{p_k}{r_k}\boldv_k^{t+1}\left( \sum_{C\in\calC}\pi(C) \sum_{S \in C} f^r_{CS} \mathds{1}_{\{k\in S \}}\right).
\end{align}
Here the term in parenthesis is, by definition,
   
 \[
  r_k = \sum_{k=1}^N \frac{p_k}{r_k}\boldv_k^{t+1} r_k
            = \sum_{k=1}^Np_j\boldv_{t+1}^j=\avgv^{t+1}.
\]

\end{proof}

We proceed by introducing a key lemma derived in \cite{li2019convergence}, characterizing convergence for the full client participation case, and then utilize it to prove Theorem~\ref{thm:importance_convergence}. 

Notice that the learning rate depends on round $t$, $t=0,...,T$ and epoch $i$ $i=1,...,E$.





\begin{lemma}
\label{lemma:full_batch_gradient}
\begin{equation}
    \mathbb{E}[\|\avgw^{t+1} - \boldw^*\|^2] \leq (1-\eta_{(t,E)}^2\mu)\mathbb{E}[\|\avgw^{t}-\boldw^* \|^2] + \eta_{(t,E)}^2\text{Var}_1,
\end{equation}

where $$\text{Var}_1= \sum^N_{k=1}p_k^2\sigma_k^2 + 6L\Gamma + 8(E-1)^2G^2.$$
\end{lemma}

\begin{proof}
This result follows from the first part of Theorem 1 in \cite{li2019convergence}, showing the convergence of \fedavg with full client participation (i.e., no client sampling).
\end{proof}

\begin{theorem*}[Theorem~\ref{thm:importance_convergence}]
Instate the settings of \cref{lemma:bound-variance}.
Let $\boldw^*$ denote the solution to the optimization problem (\ref{eq:objective_function}), and $L,\mu=O(1)$.  Define $\gamma = \max \{8\frac{L}{\mu}, E\}$, and assume learning rate $\eta_{(t,i)} = \frac{2}{\mu(\gamma+(tE+i))}$. 
Then by setting \textsc{ClientOpt} to SGD and \textsc{ServerOpt}$(\avgw^t, \Delta^{t+1}) = \avgw^t+\Delta^{t+1}$, the model $\avgw^T$ produced by Algorithm~\ref{alg:gs_fedavg} with policy $\bm{f}^{\boldr}$ satisfies
\begin{equation}
 \mathbb{E}[F(\avgw^T)] - F^* =O \left( \frac{1}{TE+\gamma} \left(  \|\avgw^1 - \boldw^* \|^2
 +\sum^N_{k=1}p_k^2\sigma_k^2 + 6\Gamma + 8(E-1)^2G^2 + \sigma_T^2(\bm{f}^{\boldr}) \right)\right),
\end{equation}
where $\Gamma = F^* - \sum_{k=1}^Np_kF_k^*$ denotes the local-global objective gap\footnote{Local-global objective gap quantifies data heterogeneity: for i.i.d. data, $\Gamma \to 0$ as the number of samples grows, while a large $\Gamma$ indicates a high degree of heterogeneity \cite{li2019convergence, cho2020client}.}, $F^*$ and $F_k^*$ are the minimum values of $F$ and $F_k$, respectively, and  $\sigma^2_T(\bm{f}^{\boldr})$ (\cref{lemma:bound-variance}) captures the variance induced by client sampling.
\end{theorem*}

\paragraph{A brief outline of the upcoming proof.}
We start by expanding $\| \avgw^{t+1} - \boldw^* \|^2$, a term that measures the distance to the optimum, and bound it by the term characterizing convergence of the full participation scheme \cite{li2018federated} plus an additional variance term emerging due to client sampling (computed in Lemma~\ref{lemma:variance}). We then invoke a standard inductive argument to express $\| \avgw^{t+1} - \boldw^* \|^2$ in terms of $\| \avgw^{1} - \boldw^* \|^2$ and, noting smoothness, finally bound $\mathbb{E}[F(\boldw_T)] - F^*$.

\begin{proof}

\begin{equation}
\label{eq:first-expansion}
      \| \avgw^{t+1} - \boldw^* \|^2 = \| \avgw^{t+1} - \avgz^{t+1}  + \avgz^{t+1} - \boldw^*\|^2 = \underbrace{\|\avgw^{t+1} - \avgz^{t+1} \|^2}_{A_1} + \underbrace{\|\avgz^{t+1}-\boldw^* \|^2}_{A_2} + \underbrace{2 \langle\avgw^{t+1} -\avgz^{t+1}, \avgz^{t+1} - \boldw^*\rangle}_{A_3}.
\end{equation}

Based on Lemma~\ref{lemma:unbiased}, we know $\Delta^{t+1}$ is unbiased, thus $\Expec{}{ \Delta^{t+1}} = \avgv^{t+1}$. 

We use this fact to prove that $A_3 = 0$ as follows:

$$\Expec{}{\langle\avgw^{t+1} -\avgz^{t+1}, \avgz^{t+1} - \boldw^*\rangle} = \Expec{}{\langle\avgw^{t} + \Delta^{t+1} -\avgw^{t}-\avgv^{t+1}, \avgz^{t+1} - \boldw^*\rangle} = \Expec{}{\langle \Delta^{t+1} -\avgv^{t+1}, \avgz^{t+1} - \boldw^*\rangle} = 0.$$

Now, $A_1$ can be bounded using \cref{lemma:bound-variance}, since $\|\avgw^{t+1} - \avgz^{t+1} \|^2 = \| (\avgw^t+\Delta^{t+1}) - (\avgw^t+\avgv^{t+1}) \|^2= \|\Delta^{t+1} - \avgv^{t+1}\|^2$.

Define 

$$\text{Var}_1 = \sum^N_{k=1}p_k^2\sigma_k^2 + 6\Gamma + 8(E-1)^2G^2, \quad \text{Var}_2 := \sigma^2_t(\bm{f}^{\boldr}) = \frac{1}{\eta_{(t,E)}^2}\Expec{}{|| \Delta^{t+1} - \avgv^{t+1} \|^2}.$$

Then by replacing Lemmas~\ref{lemma:full_batch_gradient} and \ref{lemma:variance} in \cref{eq:first-expansion} we have that

\begin{equation*}
    \mathbb{E}\| \avgw^{t+1} - \boldw^* \|^2 \leq(1-\eta_{(t,E)}\mu)\mathbb{E}\|\avgw^t-\boldw^* \|^2 + \eta_{(t,E)}^2( \text{Var}_1+\text{Var}_2).
\end{equation*}

Thanks to \cref{lemma:unbiased}, we find a similar expression to the one in \cite{li2019convergence}, but with different constants coming from different client sampling variance in \cref{lemma:bound-variance}. The rest of the proof then follows standard techniques, e.g., see \cite{li2019convergence}. We repeat those steps for the sake of completeness. 

Let $\beta>\frac{1}{\mu}$, $\gamma>0$, and define $\eta_{(t,i)} = \frac{\beta}{(t-1)E+i + \gamma}$ such that $\eta_{(1,1)}<\min\{\frac{1}{\mu}, \frac{1}{4L}\}$, and $\eta_{(t,1)} \leq 2\eta_{(t,E)}$. 

As a standard technique, we show by induction in $t$ that $\mathbb{E}\| \avgw^{t} - \boldw^* \|^2 \leq \frac{v}{\gamma+tE}$, where 

$$v=\max \{ \frac{\beta^2(\text{Var}_1+\text{Var}_2)}{\beta\mu-1}, (\gamma+1) \|\avgw^1 - \boldw^* \|^2 \}.$$ 

 This holds for $t=1$ trivially, from the definition of $v$. Now assume the claim holds for $t$; starting from the above equation, we have that 
\begin{align}
  \nonumber  \mathbb{E}\| \avgw^{t+1} - \boldw^* \|^2 &\leq(1-\eta_{(t,E)}\mu)\mathbb{E}\|\avgw^t-\boldw^* \|^2 + \eta_{(t,E)}^2( \text{Var}_1+\text{Var}_2) \\
 \label{eq:induction-step}   &\leq \left( 1-\frac{\beta\mu}{tE+\gamma}\right)\frac{v}{tE+\gamma} + \frac{\beta^2(\text{Var}_1+\text{Var}_2)}{(tE+\gamma)^2}\\
\label{eq:add-0-v}    &= \frac{tE+ \gamma-E}{(tE+\gamma)^2}v + \left[ \frac{\beta^2(\text{Var}_1+\text{Var}_2)}{(tE+\gamma)^2} - \frac{\beta\mu -E}{(tE+\gamma)^2}v\right] \\
 \label{eq:recursive-error-bound} &  \leq \frac{v}{tE+\gamma+E} =\frac{v}{(t+1)E+\gamma} .
\end{align}

where \cref{eq:induction-step} follows by induction step and definition of $\eta_{(t,E)}$, \cref{eq:add-0-v} from adding and substracting $\frac{Ev}{(tE+\gamma)^2}$, and the last step by noticing that

\[
    \frac{tE+\gamma-E}{(tE+\gamma)^2} = \frac{(tE+\gamma-E)(tE+\gamma+E)}{(tE+\gamma)^2(tE+\gamma+E)}
    \leq \frac{1}{tE+\gamma+E}.
\]

Finally, by smoothness of $F$, 

\begin{equation}
\label{eq:final-risk}
    \mathbb{E}[F(\avgw^t)] - F^* \leq \frac{L}{2}\mathbb{E}\| \avgw^{t} - \boldw^* \|^2 \leq \frac{v}{\gamma+tE}.
\end{equation}

Setting $ \beta=\frac{2}{\mu}$, $\kappa=\frac{L}{\mu}$ and $\gamma = \max\{8\kappa,E \}-1$, and using Lemma~\ref{lemma:variance} to compute $\text{Var}_2$, we obtain the desired result. 
\end{proof}

\subsection{Proof of \cref{lemma:bound-variance}: Bounded client sampling variance}
\label{sec:proofs_lemmas}

For clarity of the proof of \cref{lemma:bound-variance}, we first introduce the following lemma on the inner product of local models.
\begin{lemma}
\label{lemma:update-norm}
At round  $t$ for any pair of clients $i$ and $j$,

\begin{equation}
  \Expec{}{\langle \boldv_i^t, \boldv_j^t\rangle } \leq 4E^2G^2\eta_{(t,E)}^2,
\end{equation}
where the expectation is taken over the samples in  local SGD steps.

\begin{proof}

Recall that $\boldv_i^t$ represents the user $i$'s update after training locally for $E$ epochs  starting with model $\avgw^{t-1}$. Then,

\[
\|\boldv_i^t \|^2  = \|\boldw_i^{t,E} - \boldw_i^{t,0} \|^2
    =  \| \sum_{\ell=0}^{E-1} \eta_{(t, \ell)} \nabla f_k(\boldw_i^{t,\ell}, \xi_{\ell})\|^2
    \leq E\sum_{\ell=0}^{E-1}\eta^2_{t,0} \|\nabla f_k(\boldw_i^{t,\ell})\|^2
    \leq 4\eta_{(t,E)}^2 E^2G^2,
\]
where we used the Jensen inequality and the fact that $\eta_{(t,\ell)}$ is decreasing
(i.e., $\eta_{(t, 0)}\leq 2\eta_{(t,\ell)}$ for $\ell \leq E$).
Now, since $\langle \boldv_i^t, \boldv_j^t\rangle \leq \max_k  \|\boldv_k^t \|^2$, the result follows. 

\end{proof}
\end{lemma}
Next, we define the random vector $X^{\boldr} \in \{0,1\}^N$ that indicates which clients are selected, and specify its moments. 
\begin{lemma}
\label{lemma:selection-vector} Let $X^{\boldr} \in \{0,1\}^N$ be such that its $i^{th}$ component takes on value $X^{\boldr}_{i} = 1$ if client $i$ is selected at time $t$, and 0 otherwise. Let $\Sigma$ denote the covariance matrix of $X^{\boldr}$. Then $\Expec{}{ X} = \boldr $ and $\text{Var}(X_i) = r_i(1-r_i)$.  
\end{lemma}
\begin{proof}
This follows trivially from the observation that $X_i$ is a Bernoulli random variable with parameter $r_i$. 
\end{proof}
\begin{lemma}
\label{lemma:variance}
Let $\boldr \in \calR^+$ be an achievable sampling rate, $\bm{f}^{\boldr}$ denote a static configuration-dependent sampling policy achieving $\boldr$, and $X^{\boldr}$ the corresponding selection random vector with covariance $\bSig$. Then for $t =1,...,T$,
\begin{equation}
\label{eq:variance}
    \sigma_t^2(\bm{f}^{\boldr}):= \frac{1}{\eta_{(t,E)}^2}\Expec{} {\|\Delta^{t} - \avgv^{t} \|^2}=  \frac{1}{\eta_{(t,E)}^2}\tr(\bY_{t}\bY_{t}^T \bSig),
\end{equation}
where vector $\frac{p_k}{r_k}\boldv^t_k$ is the $k^{th}$ row of matrix $\bY_t \in \mathbb{R}^{N\times p}$.
\end{lemma}



\begin{proof}
Below $\| \cdot \|$ denotes the $\ell_2-$norm. Using the variance formula and that $\Expec{}{\Delta^{t}} = \avgv^t$ by \cref{lemma:unbiased},
\begin{equation}
   \label{eq:expansion-variance} \Expec{S}{ \|\Delta^t - \avgv^t \|^2}  = \Expec{}{\| \Delta^t\|^2} - \|\avgv \|^2.
\end{equation}
Let us focus on the first term:
\begin{align*}
       \Expec{}{\| \Delta^t\|^2}&= \sum_{C \in \calC}\pi(C)\sum_{S \in C}f_{CS} \| \sum_{k\in S} \frac{p_k}{r_k} \boldv_k^{t}\|^2
       = \sum_{C \in \calC}\pi(C)\sum_{S \in C}f_{CS} \sum_{i,j \in S} \frac{p_ip_j}{r_ir_j}\langle \boldv_i,\boldv_j \rangle \\
       &= \sum_{i,j=1}^N \sum_{C \in \calC}\pi(C)\sum_{S \in C}f_{CS} \frac{p_ip_j}{r_ir_j}\langle \boldv_i,\boldv_j \rangle \mathds{1}_{i\in S} \mathds{1}_{j\in S}
       =\sum_{i,j=1}^N \sum_{C \in \calC}\pi(C)\sum_{S \in C}f_{CS} \frac{p_ip_j}{r_ir_j}\langle \boldv_i,\boldv_j \rangle X_iX_j.
\end{align*}

From the defition of $\bY_t$ and by reorganizing,

\[
\Expec{}{\| \Delta^t\|^2}= \Expec{}{X^T\bY_t\bY_t^TX}.
\]

Introducing $\mathbf{B}=\bY_t\bY_t^T$,

\begin{align}
 \nonumber \Expec{}{\| \Delta^t\|^2}  &=  \Expec{}{ \sum_{i,j}b_{ij}X_iX_j}
  = \sum_{i,j}b_{ij} \Expec{}{X_iX_j} \\
  \label{eq:cov-step}  &=  \sum_{i,j}b_{ij} (\mathbf{\Sigma}_{ij} + r_ir_j)\\
 \nonumber &= \sum_{i}[\mathbf{B\Sigma}]_{ii} + \boldr^T\mathbf{B}\boldr \\
 \label{eq:cancel-r-step} & = \tr(\mathbf{Y_tY_t^T \Sigma}) + \|\avgv^t \|^2,
\end{align}
where Eq.~\ref{eq:cov-step} follows by the covariance formula and Eq.~\ref{eq:cancel-r-step} follows due to cancellation in denominator of $\mathbf{B}=\bY_t\bY_t^T$. After combining this with \cref{eq:expansion-variance}, the term $\| \avgv^t\|^2$ cancels and we obtain the desired result. 
\end{proof}

\begin{lemma*}[\cref{lemma:bound-variance}]
Suppose Assumptions~\ref{assum:constraint_process}-\ref{assum:bound_gradient_norm} hold. Let $\boldr \in \calR$ be an achievable sampling rate under the system configuration determined by distribution $\pi$, and $\bm{f}^{\boldr}$ denote a static configuration-dependent sampling policy achieving  $\boldr$. 

Define the client sampling variance $\sigma^2_t(\bm{f}^{\boldr}): = \frac{1}{\eta_t^2}\Expec{S_t}{\|\Delta^{t}-\avgv^t \|^2}$ where $\avgv^t = \sum_{k=1}^Np_k\boldv^k_{t}$ is the update at time $t$ with full client participation. Then

\begin{equation}
    \sigma_t^2(\bm{f}^{\boldr}) \leq 4E^2G^2  (\sum_{k=1}^N\frac{p_k}{r_k} - 1).
\end{equation}
 
Furthermore, if client availabilities are uncorrelated or negatively correlated, then there exists a policy $\bm{f}^r$ such that 
\begin{equation}
    \sigma_t^2(\bm{f}^{\boldr}) \leq 4E^2G^2 \left( \sum_{k=1}^N\frac{p_k^2}{r_k} + \sum_{k=1}^N p_k^2 \right).
\end{equation}
\end{lemma*}


\begin{proof}
 For the first part, taking expectation over the independent local SGD sampling and over the random set of clients  $S$, 
 
 \begin{align*}
     \Expec{}{ \| \Delta^t\|^2} &=     \Expec{}{\| \sum_{k\in S} \frac{p_k}{r_k} \boldv^t_k]\|^2}
     = \Expec{}{ \sum_{i,j\in S} \frac{p_ip_j}{r_ir_j}\langle\boldv_i^t,\boldv_j^t \rangle}
     = \sum_{i,j=1}^N \frac{p_ipj}{r_ir_j}\Expec{}{ \langle\boldv_i^t,\boldv_j^t \rangle }P(i,j \in S)\\
     & \leq 4E^2G^2\eta_{(t,E)}^2\left( \sum_{i=1}^N\frac{p_i^2}{r_i^2}P(i \in S) + \sum_{i=1}^N\sum_{j=1, j \neq i}^N \frac{p_ip_j}{r_ir_j}P(i,j \in S)\right).
 \end{align*}
 Given that $P(i,j \in S) \leq P(i\in S)$ and $P(i\in S) =r_i$,
 
\begin{align*}
 \Expec{}{ \| \Delta^t\|^2} & \leq 4E^2G^2\eta_{(t,E)}^2 \left(\sum_{i=1}^N\frac{p_i^2}{r_i} + \sum_{i=1}^N\sum_{j=1,j\neq i}^N \frac{p_ip_j}{r_ir_j}P(j \in S)\right) \\
        & \leq 4E^2G^2\eta_{(t,E)}^2 \left(\sum_{i=1}^N\frac{p_i^2}{r_i} + \sum_{i=1}^N\frac{p_i}{r_i}\sum_{j=1, j\neq i}^N p_j\right) \\
        & \leq 4E^2G^2\eta_{(t,E)}^2 \left( \sum_{i=1}^N\frac{p_i^2}{r_i} + \sum_{i=1}^N\frac{p_i}{r_i}(1-p_i)\right) \\
        &\leq 4E^2G^2\eta_{(t,E)}^2 \sum_{i=1}^N\frac{p_i}{r_i}.
\end{align*}

Therefore,
\[
\Expec{}{ \| \Delta^t - \avgv^t\|^2} = \Expec {}{\| \Delta^t \|^2} - \|\avgv^t \|^2\\
\leq 4E^2G^2 \eta_{(t,E)}^2 (\sum_{i=1}^N\frac{p_i}{r_i} - 1),
\]
and the result follows. 

For the second part of the lemma, consider a policy $\bm{f}$ with rate $\boldr$ that at time $t$ selects $S$ as

$$S_t \in\arg\max_{S \in C_t } -\nabla H(\boldr) \cdot \mathds{1}_S.$$

Let $u_k$ denote the $k$-th largest utility value, where the individual utilities are defined by vector $-\nabla H(\boldr)$.  W.l.o.g. assume $u_i<u_j$ if $i<j$, and let $A_k$ be the (random) number of users $i$ with the utility less than $u_k$. Let $K$ be a bound on the set size $S$. Let $i,j$ be two users, $i<j$; then $u_i<u_j$. 
Now, since $i,j$ are uncorrelated or negatively correlated, $P(i,j \text{ are available}) \leq P(j \text{ is available})P(i \text{ is available})$. Therefore,
\[
P(j \in S | i \in S ) =  P(j \text{ is available})P(A_k-1 < K)
\leq  P(j \text{ is available})P(A_k  < K) = P(j \in S  ).
\]
Note that from the definition of conditional probability, it follows that the sampling is also uncorrelated since 
\[
    P(i,j \in S) = P(i\in S)P(j \in S | i \in S ) = P(i\in S)P(j\in S)
 = \Expec{}{X_i}\Expec{}{X_j}
 = r_ir_j.
\]
Therefore, we have that $\Sigma_{ij} = \Expec{}{X_iX_j} - \Expec{}{X_i}\Expec{}{X_j} = P(i,j \in S ) - r_ir_j \leq 0$. From Eq.~\ref{eq:cov-step}, we have that $$ \tr(\mathbf{YY^T \Sigma}) = \sum_{i,j}\frac{p_ip_j}{r_ir_j}\langle \boldv_i^t,\boldv_j^t \rangle\mathbf{\Sigma}_{ij};$$
since $\mathbf{\Sigma}_{ij} \leq 0$ for $i\neq j$, 
\begin{align*}
  \tr(\mathbf{YY^T \Sigma}) &= \sum_{i,j}\frac{p_ip_j}{r_ir_j}\langle \boldv_i^t,\boldv_j^t \rangle\mathbf{\Sigma}_{ij}
    \leq 4E^2G^2 \eta_{(t,E)}^2\sum_{i}\frac{p_i^2}{r^2_i}\Sigma_{ii} + \sum_{i\neq j}\frac{p_ip_j}{r_ir_j}\Sigma_{ij} \\
    &\leq 4E^2G^2\eta_{(t,E)}^2 \sum_{i}\frac{p_i^2}{r^2_i}(r_i(1-r_i))
    \leq  4E^2G^2\eta_{(t,E)}^2 \left( \sum_{i}\frac{p_i^2}{r_i} +  \sum_{i} p_i^2 \right),
\end{align*}
where the first inequality follows from \cref{lemma:update-norm} and breaking the sum on diagonal and non diagonal terms. The second inequality follows by dropping negative terms and replacing the variance value $\Sigma_{ii} = r_i(1-r_i)$ for $X_i$ (\cref{lemma:selection-vector}), while the last one follows simply by expanding the previous term.

\end{proof}

\subsection{Proof of Theorem~\ref{thm:sampling_rate_convergence} }
\begin{theorem*}[Theorem~\ref{thm:sampling_rate_convergence}]
Let $\boldr^{\beta}(t)$ be defined by \cref{alg:gs_fedavg}, following equations~(\ref{eq:gradient_sched}) and (\ref{eq:update_service_rate}). Let $V \subset \mathbb{R}_{+}^N$ be a bounded set, $\epsilon>0$, and let $\boldr^*$ denote the minimizer of the variance function $H(\boldr) = \sum_{k=1}^n \frac{p_k^2}{r_k}$ over $\calR$. Then for $T>0$, depending on $\epsilon$ and $V$,
\begin{equation*}
   \lim_{\beta \downarrow 0}\quad \sup_{\boldr^{\beta}(0) \in V, t>T/\beta}P [\| r^{\beta}(t) - r^*\|>\epsilon]=0.
\end{equation*}
\end{theorem*}
\begin{proof}
Our proof follows the stochastic approximation analysis in \cite{stolyar2005asymptotic} that establishes an attraction property for Fluid Sample Paths (FSPs), which are limiting trajectories of the generalized versions of the processes considered in our manuscript. The idea behind the proof is that as $t$ grows, we can study the limiting trajectories (i.e., FSPs) $x=(x(t), t \geq 0)$ of the process $r(t/\beta)$. Since $-H(\boldr)$ is a convex function, it follows from Theorem 4 in \cite{stolyar2005asymptotic} that for an arbitrary initial state $x(0)$, $x(t) \to \boldr^*$ as $t \to \infty$. More concretely, from Theorem 3 in \cite{stolyar2005asymptotic} it follows that as $\beta \downarrow 0$, a limit of sequence $\{\boldr^{\beta}\}$  considered in Section~\ref{sec:model} of our paper is a process with sample paths being FSPs $x$ with probability 1. 

Assume $V \subset[0,a]^N$ such that $a>\bar{\mu}$, where $\bar{\mu} = \max_{C\in \calC,S \in S} \frac{1}{|S|}$. Let $\epsilon >0$ and $\delta>0$. By the above result on FSPs (i.e., by Theorem 4 in \cite{stolyar2005asymptotic}), we can find $T$ large enough such that $\|x(t) - \boldr^*\| \leq \epsilon$ uniformly for $t$ in the interval $[T, T+\delta]$. 
Combining this result with the continuous mapping theorem  \cite{billingsley2013convergence}, we obtain
\begin{equation*}
     \lim_{\beta \to 0} \sup_{r^{\beta}\in V} P (\sup_{t \in [T,T+\delta]}\|x(t) - \boldr^*\| >\epsilon) = 0.
\end{equation*}
Also, notice that for all $t$, $r^{\beta}_i(t) \leq \max\{ \bar{\mu}, \boldr^{\beta}_i(0)\}$ element-wise, where  $\bar{\mu} = \max_{C\in \calC,S \in S} \frac{1}{|S|}$ by construction. Then for any $\tau \geq 0$ we can re-start the process, implying
\begin{equation*}
     \lim_{\beta \to 0} \sup_{r^{\beta}\in V} \sup_{\tau \geq 0} P(\sup_{t \in [\tau+T,\tau+T+\delta]}\|x(t) - \boldr^*\| >\epsilon) = 0
\end{equation*}
and thus establishing the desired result.
\end{proof}

\section{Experiments details}
\label{sec:app_experiments}

\paragraph{Hyperparameter tuning.}
We set the learning rate on Shakespeare and CIFAR100 according to the optimal values found in \cite{reddi2020adaptive}. For the synthetic dataset we use the learning rate tuned in \cite{li2018federated}, $\eta=0.01$. For Shakespeare we use mini-batches of size 4, and mini-batches of size 20 for the remaining datasets. Following literature, in all the experiments we use $\beta = O(1/T) = 0.001$.

\paragraph{Machines.} We ran our experiments on AMD Vega 20 (ROCm) cards. One rounds of training in Fig. 1 require 8 GPU seconds for CIFAR100,
67 GPU seconds for Shakespeare, and 0.57 GPU seconds  for Synthetic(1,1).

\subsection{Datasets}
\paragraph{Synthetic dataset.}
We generate this data by taking $10^4$ samples $X_i \in \mathbb{R}^{100} \sim \mathcal{N}(0,I_{100})$. Moreover, we generate $\beta\sim \mathcal{N}(0,I_{100})$ and, finally, set labels $y_i = round(X_i^T\beta)$. The samples are split evenly among $100$ clients. 

\paragraph{Skakespeare. }
Each client's dataset is restricted to have at most 128 sentences, and is split into training and validation sets.  Following the previous work with this dataset \citep{reddi2020adaptive}, we use a build vocabulary with 86 characters contained in the text, and 4 characters representing padding, out-of-vocabulary, beginning, and end of line tokens. We use padding and truncation to enforce 20 word sentences, and represent them with index sequences corresponding to the vocabulary words, out of vocabulary words, beginning and end of sentences.
\subsection{Federated Datasets Details}
Number of clients and total number of samples is summarized in \cref{table:datasets}.

\begin{table}
\caption{Datasets}
\label{table:datasets}
\centering
\begin{tabular}{@{}lcc@{}}
\toprule
Dataset       & Users & Samples \\ \midrule
Synthetic        & 100  & 60 K    \\
CIFAR100     &   500     &    50 K   \\
Shakespeare & 715 & 16 K   \\ \bottomrule
\end{tabular}
\end{table}

\subsection{Models}
\label{sec:app:architectures}
We train a recursive neural network for the next character prediction that first embeds characters into an 8-dimensional space, followed by 2 LSTMs and finally a dense layer. The architecture is presented in Table \ref{tab:ShakespeareArchit}.

 ResNet-18 architecture can be found in \cite{he2016deep}, where we replace batch normalization by group normalization \cite{wu2018group} as in \cite{reddi2020adaptive}.

\begin{table}
\caption{ Shakespeare next character prediction model architecture}
\label{tab:ShakespeareArchit}
\centering
\begin{tabular}{@{}lccc@{}}
\toprule
\multicolumn{1}{c}{\textbf{Layer}} & \textbf{Output} & \textbf{\begin{tabular}[c]{@{}c@{}}\# Trainable \\ parameters\end{tabular}} & \textbf{Activation}  \\ \midrule
Input                              & 80              & \multicolumn{1}{l}{}                                                        & \multicolumn{1}{l}{} \\
Embedding                          & (80,8)         & 720                                                                      &                      \\
LSTM                               & (80,256)        & 271360                                                                     &                      \\
LSTM                               & (80,256)         & 525312                                                                    &                      \\
Dense                              & (80,90)      & 23130                                                                     & Softmax              \\ \midrule
\multicolumn{1}{c}{\textbf{Total}} &                 & 4,050,748                                                                 &                      \\ \bottomrule
\end{tabular}
\end{table}

\subsection{Availability models}
\label{sec:app:avail_model}

For the {\em Home-devices} model $t_k\sim lognormal(0,0.5)$, while for the {\em Smartphones} $t_k\sim lognormal(0,0.25)$. The sine wave is defined by $f(t) = 0.4\sin(t)+0.5$ and we sample at times $t =\frac{2\pi j}{24}$ for $j=1, ...,24$.

\subsection{Loss and accuracy values under independent availability model.}
\cref{fig:stable_convergence_loss} shows that \vargrad exhibits a much more stable convergence for all data sets and achieves a smaller loss value.  \cref{fig:averaged-accuracy} also proves \feast is more resilient: it maintains similar performance across 3 independent runs. 

\begin{figure}[h]
       \centering
    \includegraphics[width=\textwidth]{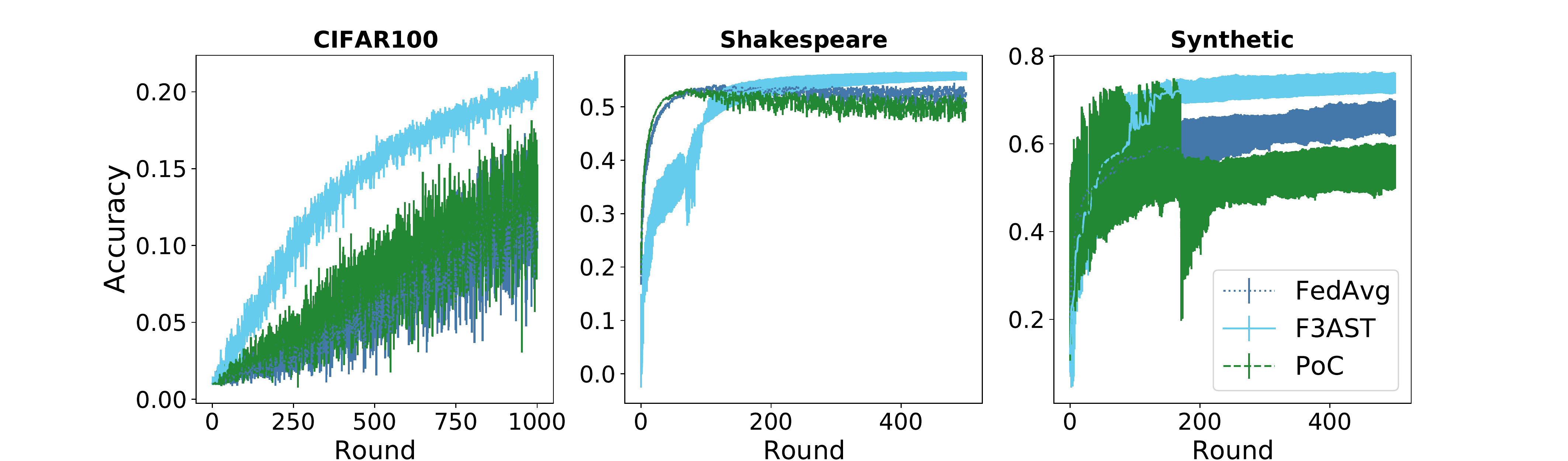}
    \caption{Test per-sample accuracy of different algorithms over all data sets (averaged over three runs). In all cases \vargrad converges to a model with higher accuracy. Furthermore, \vargrad stabilizes while \fedavg and \poc are not able to adapt to the time-varying environment. \feast also provides the smallest variance across independent runs.}
        \label{fig:averaged-accuracy} 
\end{figure}

\subsection{Varying the communication constraint.}

\cref{fig:num_clients} shows the test accuracy during training for the three algorithms (\vargrad, \fedavg and \poc\!). We observe that \vargrad achieves equal or higher performance than competing methods across all communication levels. Note that \poc stagnates at a similar accuracy in all cases; we believe this is due to the inherent bias in the algorithm due to top-k loss based sampling, as reported by the authors. It is possible that certain groups of clients are never selected by this policy. It is interesting to note that the gap between \vargrad and two baselines widens as the number of selected clients increases. Indeed, as the number of users grow, a configuration-dependent policy becomes much harder to facilitate since the number of possible selections grows exponentially with $K$. Nevertheless, the greedy nature of \vargrad allows it to keep selecting the set of users that maximizes marginal utility and achieves a balanced sampling rate under the availability model. The other two policies, however, do not track previous selections of users and thus may end up over-selecting available users rather than exploring the full pool of devices. 

\begin{figure*}[h]
       \centering
    \includegraphics[width=\textwidth]{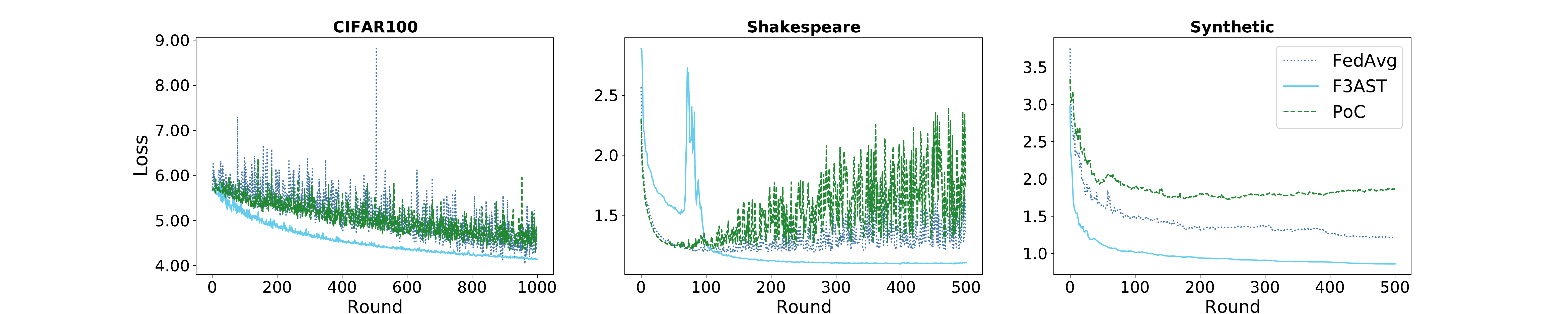}
    \caption{Test per-sample loss of different algorithms over all data sets. In all cases \vargrad converges to a model with smaller objective value. Furthermore, \vargrad stabilizes while \fedavg and \poc are not able to adapt to the time-varying environment.}
        \label{fig:stable_convergence_loss} 
\end{figure*}

 \begin{figure*}[h]
    \centering
    \includegraphics[width=\textwidth]{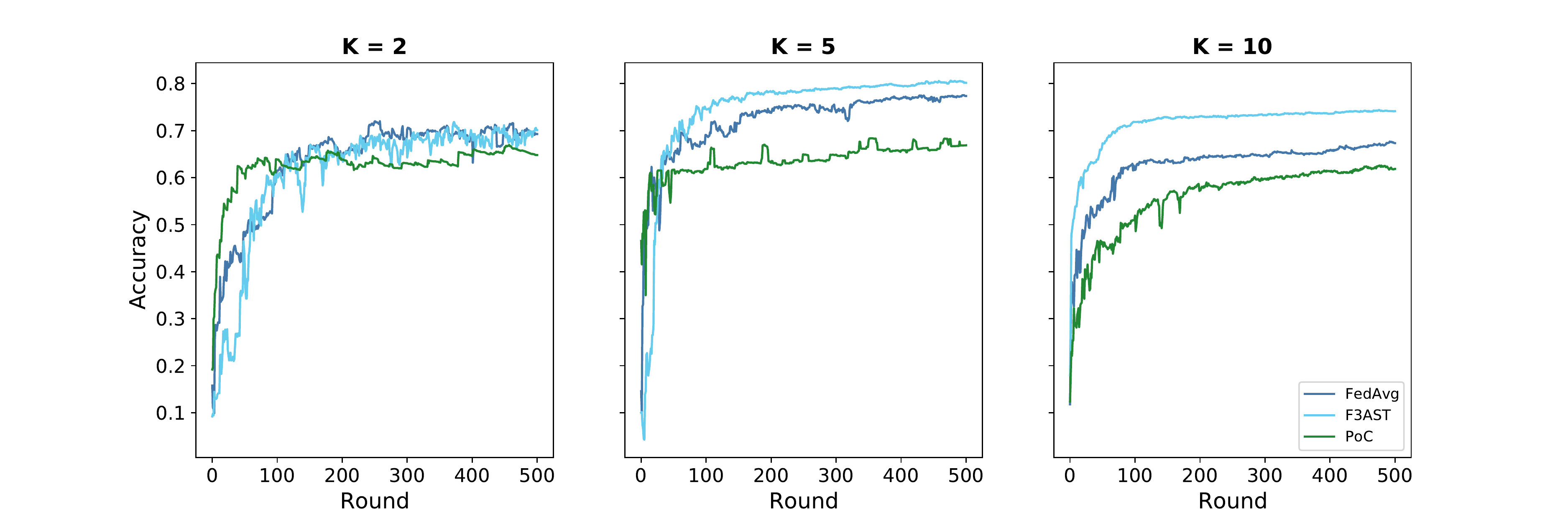}
    \caption{Impact of varying communication level $K$ in the Synthetic(1,1) dataset experiments. As the number of sampled clients increases, the gap between \vargrad and competing methods widens. }
    \label{fig:num_clients}
\end{figure*}

\subsection{Additional user/sample results and accuracy visualization. }

\cref{fig:table} presents average user and sample loss over CIFAR100 and Shakespeare for all availability models. User average first computes the local test metric and then averages over users, while sample metrics computes the metric averaged over all test samples. Typically in deep learning test sample accuracy reflects better performance. However, in federated learning uniform performance across users is desired even if some users have very few datapoints. Below we show that claims in \cref{sec:experiments} are also valid for user-averaged metrics. Notice that sample accuracy and user accuracy coincide in CIFAR100 given that data is split in a balanced way among clients so we omit these results.

\begin{figure*}
\centering
\begin{subfigure}{0.45\textwidth}
    \centering
    \includegraphics[width =\textwidth]{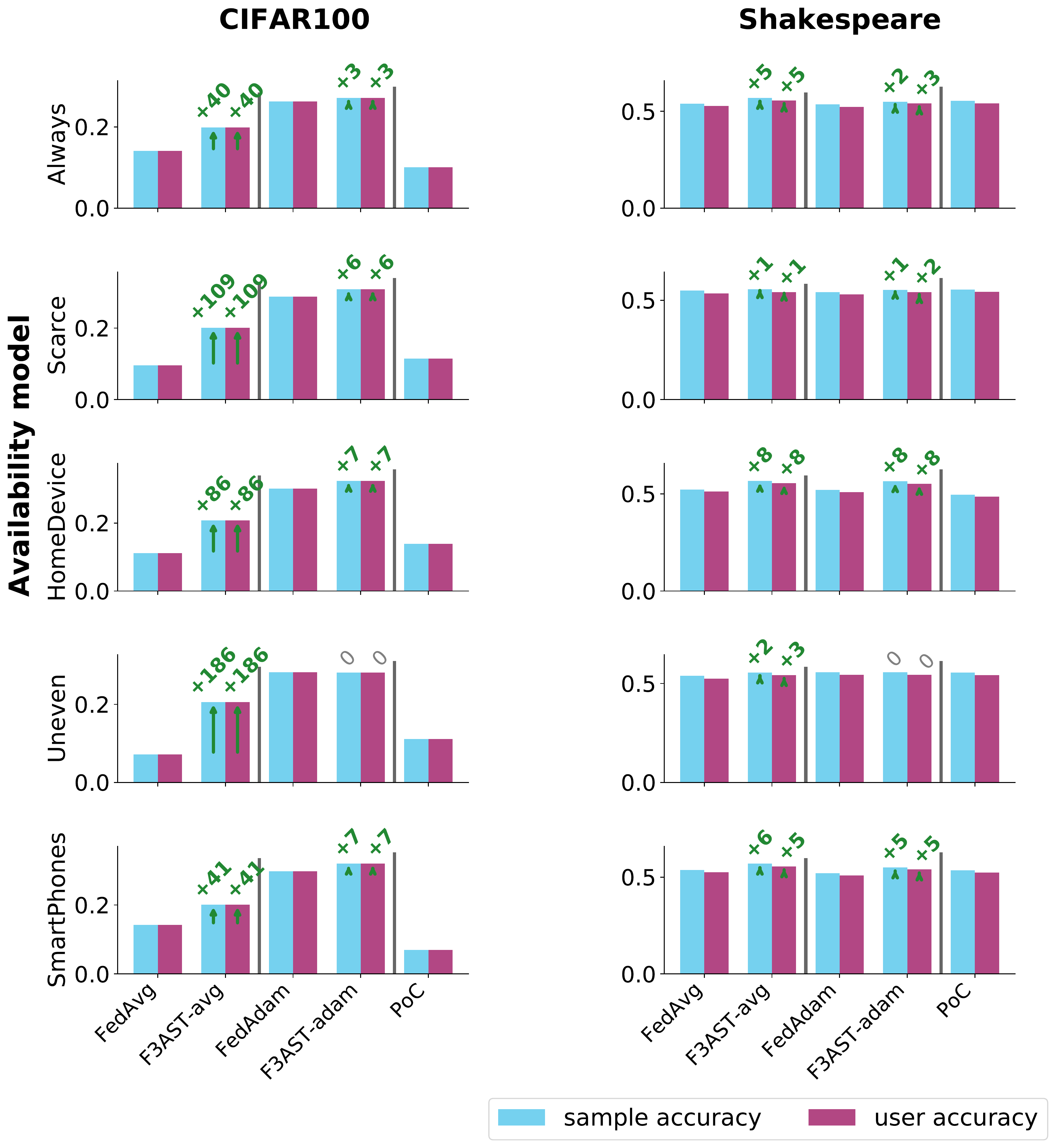}
    \caption{Test accuracy}
\end{subfigure}
\begin{subfigure}{0.45\textwidth}
    \centering
    \includegraphics[width=\textwidth]{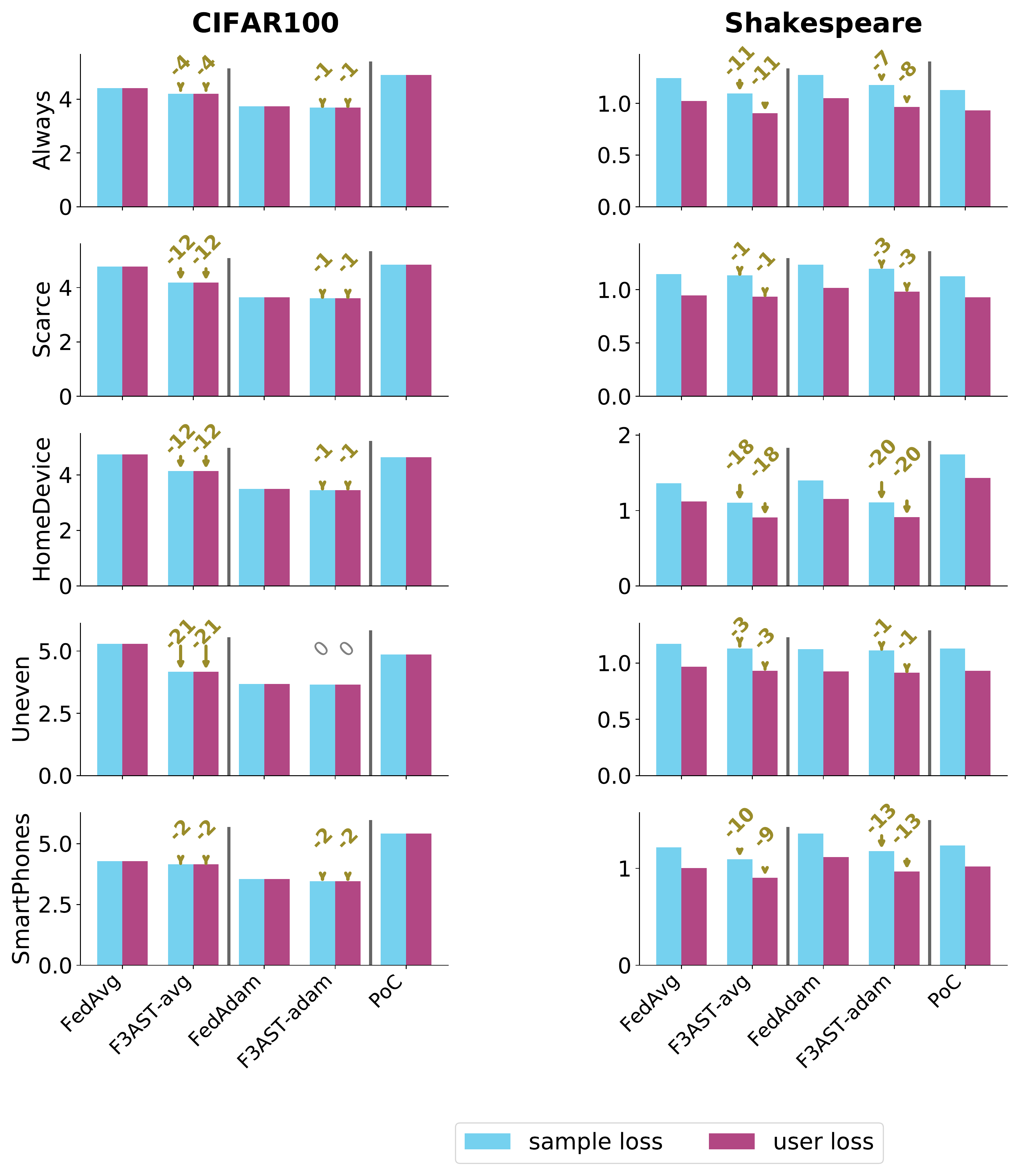}
    \caption{Train loss}
    \label{fig:loss}
\end{subfigure}
    \caption{Accuracy and relative gain of \vargrad over \fedavg and \fedadam for different availability models (rows) and datasets (columns).  \vargrad consistently improves both the unweighted and weighted accuracy.}
    \label{fig:table}
\end{figure*}

\begin{table}[]
\caption{User accuracy of all methods, and relative improvements of \vargrad over \fedavg and \vargrad+ Adam over \fedadam\!\!, for different availability models (columns) on Shakespeare (500 rounds).}
\centering
\begin{tabular}{@{}llllll@{}}
\toprule
\multicolumn{1}{r}{Availability model}      &    Always       &      Scarce     &      HomeDevice     &      Uneven     &      SmartPhones \\
Routine     &            &            &            &            &        \\
\midrule
\fedavg &             0.529     &     0.534     &     0.512     &     0.525     &     0.526 \\
\feast &             \textbf{0.556} \textcolor{cb-blue}{(+5\%)}    &     \textbf{0.542}  \textcolor{cb-blue}{(+1\%)}   &     \textbf{0.556}  \textcolor{cb-blue}{(+8\%)}   &     0.543  \textcolor{cb-blue}{(+3\%)}   &     \textbf{0.556} \textcolor{cb-blue}{(+5\%)} \\
\fedadam &             0.524     &     0.529     &     0.509     &     0.544     &     0.510 \\
\feast+Adam &             0.541  \textcolor{cb-blue}{(+3\%)}   &     0.541  \textcolor{cb-blue}{(+2\%)}     &     0.552  \textcolor{cb-blue}{(+8\%)}      &     \textbf{0.544}    \textcolor{gray}{(0\%)}   &     0.540  \textcolor{cb-blue}{(+5\%)}  \\
\poc &             0.541     &     \textbf{0.543}     &     0.486     &     0.543     &     0.524\\
\bottomrule
\end{tabular}
\end{table}

\begin{table}[]
\caption{User loss of all methods, and relative improvements of \vargrad over \fedavg and \vargrad+ Adam over \fedadam\!\!, for different availability models (columns) on Shakespeare (500 rounds).}
\centering
\begin{tabular}{@{}llllll@{}}
\toprule
\multicolumn{1}{r}{Availability model}      &    Always       &      Scarce     &      HomeDevice     &      Uneven     &      SmartPhones \\
Routine     &            &            &            &            &        \\
\midrule
\fedavg &             1.02     &     0.95     &     1.12     &     0.97     &     1.00 \\
\feast &             \textbf{0.90}  \textcolor{cb-lilac}{(-11\%)}   &     \textbf{0.93} \textcolor{cb-lilac}{(-1\%)}    &     \textbf{0.91}  \textcolor{cb-lilac}{(-18\%)}   &     0.93  \textcolor{cb-lilac}{(-3\%)}   &     \textbf{0.90} \textcolor{cb-lilac}{(-9\%)}\\
\fedadam &             1.05     &     1.02     &     1.16     &     0.93     &     1.12 \\
\feast+Adam &             0.97  \textcolor{cb-lilac}{(-8\%)}   &     0.98 \textcolor{cb-lilac}{(-3\%)}    &     \textbf{0.91}   \textcolor{cb-lilac}{(-20\%)}  &     \textbf{0.92}  \textcolor{cb-lilac}{(-1\%)}   &     0.97 \textcolor{cb-lilac}{(-13\%)} \\
\poc &             0.93     &     \textbf{0.93}     &     1.43     &     0.93     &     1.02 \\
\bottomrule
\end{tabular}
\end{table}

\subsection{Varying heterogeneity in the synthetic dataset. }

\begin{table}[h]
\caption{Test accuracy on the synthetic($\alpha,\alpha$) dataset under the Smartphones availability model. } \label{tab:heterogeneity}
\centering
\begin{tabular}{@{}cccccc@{}}
\midrule
 ~      &    $\alpha=0$      &      $\alpha=0.5$     &      $\alpha=1$      \\
\midrule
\fedavg &             0.72     &     0.72     &     0.68 \\ 
\feast &             0.83     &     0.75     &     0.76 \\  
\midrule
\end{tabular}
\end{table} 
In our paper we consider heterogeneity of both data and clients' availability. Data heterogeneity arises from unbalanced amounts of data generated from heterogeneous distributions. \feast aims to address data heterogeneity by selecting clients whose participation is far from the desired rate. \cref{tab:heterogeneity} shows that for all values of the synthetic data heterogeneity parameter $\alpha$ (Smartphones setting), \feast is more accurate than \fedavg.


\end{document}